\documentclass{ecai}

\usepackage{epic}

\usepackage{graphicx}
\usepackage{latexsym}
\usepackage{times}
\usepackage{soul}
\usepackage{url}
\usepackage[hidelinks]{hyperref}
\usepackage[utf8]{inputenc}
\usepackage{rotating}
\usepackage[graphicx]{realboxes}
\usepackage{amsmath}
\usepackage{amsthm}
\usepackage{booktabs}
\usepackage{algorithm}
\usepackage{algorithmic}
\usepackage{amssymb}
\usepackage{xcolor}
\usepackage{colortbl}
\usepackage{tikz}
\usepackage{comment}
\usepackage{multicol}
\usepackage{tcolorbox}

\newenvironment{exs}[1]{\begin{trivlist}\item[]\textbf{Example \ref{#1} (Cont)}}{\end{trivlist}}

\newtheorem{prop}{Proposition}
\newtheorem{proposition}[prop]{Proposition}
\newtheorem{defn}{Definition}
\newtheorem{definition}[defn]{Definition}

\newtheorem{exmp}{Example}
\newtheorem{example}[exmp]{Example}

\newtheorem{theorem}{Theorem}

\newtheorem{property}{Property}

\newcommand{\bQ}{{\mathbf{Q}}}
\newcommand{\bL}{{\mathbf{L}}}
\newcommand{\bLw}{{{\mathbf{L_w}}}}
\newcommand{\bLa}{{\mathbf{L_c}}}
\newcommand{\bLdw}{{\mathbf{L_{dw}}}}
\newcommand{\bLda}{{\mathbf{L_{dc}}}}
\newcommand{\bLr}{{\mathbf{L_{ir}}}}
\newcommand{\bLc}{{\mathbf{L_{co}}}}
\newcommand{\bLt}{{\mathbf{L_{tr}}}}
\newcommand{\bLDT}{{\mathbf{L_{DT}}}}
\newcommand{\bLsu}{{\mathbf{L_{su}}}}

\newcommand{\F}{{\mathtt{F}}}
\renewcommand{\d}{{\mathtt{dom}}}

\newcommand{\C}{\mathtt{C}}
\newcommand{\D}{\mbox{$\cal D$}}

\newcommand{\mbb}[1]{\ensuremath\mathbb{#1}}

\definecolor{maroon}{cmyk}{0,0.87,0.68,0.32}
\begin{document}

\begin{frontmatter}

\title{Axiomatic Characterisations of Sample-based Explainers}
\author[A]{\fnms{Leila}~\snm{Amgoud}}
\author[B]{\fnms{Martin}~\snm{Cooper}\thanks{Corresponding Author. Email: cooper@irit.fr}}
\author[B]{\fnms{Salim}~\snm{Debbaoui}}

\address[A]{IRIT, CNRS, France}
\address[B]{IRIT, Toulouse University, France}
\begin{abstract}
Explaining decisions of black-box classifiers is both important and computationally challenging. 
In this paper, we scrutinize explainers that generate \textit{feature-based explanations} from samples or datasets. 
We start by presenting a set of desirable properties that explainers would ideally satisfy, delve into their  relationships, and  highlight incompatibilities of some of them. 
We identify the entire family of explainers that satisfy two key properties which are compatible with all the others. Its instances provide sufficient reasons, called \emph{weak abductive explanations}. 
We then unravel its various sub-families that satisfy subsets of compatible properties.  
Indeed, we fully characterize all the explainers that satisfy any subset of compatible properties. 
In particular, we introduce the first (broad family of) explainers that guarantee the existence of explanations  and their global consistency. We discuss some of its instances including the \textit{irrefutable explainer} and the \textit{surrogate explainer} whose explanations can be found in polynomial time.
\end{abstract}

\end{frontmatter}
\section{Introduction}  

In recent years, AI systems have demonstrated remarkable capabilities 
but this success has often come at the expense of a corresponding lack of explainability.
Indeed, deep neural networks behave effectively like black-box functions, meaning that
explaining their decisions is inherently intractable. On the other hand,
legislators have recognised the importance of providing explanations to end-users affected 
by decisions taken by AI systems~\cite{regul21,WH22}. Thus, the question is how to reconcile the
user's right to an explanation with the intractability of the corresponding computational problem.

In this paper, we concentrate on decisions taken by classifiers learnt by Machine-Learning (ML).
There have been some important successes in identifying ML models which allow either
polynomial-time explainability~\cite{AudemardBBKLM22,CarbonnelC023,COOPER2023,IzzaIM22,ICML21} 
or effectively-tractable explainability via efficient solvers~\cite{Audemard23,CIKM21,IgnatievS21,Izza021}. 
However, success of such formally-correct explaining on neural networks has been limited to relatively
small networks~\cite{joao19b}. One solution to the right-to-an-explanation/intractability-of-explaining
dilemma is the use of \textit{surrogate models}: instead of explaining the decision of classifier $\kappa$
on an instance $x$, the back-box function $\kappa$ is approximated by a simpler explainable function 
$\tilde{\kappa}$ in the neighbourhood of $x$ \cite{Ribeiro16,Ribeiro0G18}. 

Another solution is to produce explanations which are valid on a subset of feature space~\cite{Amgoud23a,Amgoud23b,ecai23}. The argument behind this endeavour is that since ML models are obtained from a dataset rather than a complete search of feature space, a dataset-based approach to explanation 
has the advantage of being a coherent echo of the techniques used in ML.
In the literature, the dataset may be, for example, the one  
used to train the ML model or any other sample of instances which are deemed general enough 
to represent the whole feature space. The most studied type of explanations is 
\textit{sufficient reasons}, called also \textit{abductive explanations}. An abductive explanation 
for assigning a class to a given instance is a subset of the instance, viewed as a set of 
(feature,value) pairs, which is sufficient to guarantee the class.


The current literature has revealed that the sample-based approach is quite challenging as it 
faces various issues. Indeed, it has been shown in \cite{Amgoud23a,ecai23} that 
dataset-based abductive explanations can be globally inconsistent in the following technical
sense: two explanations for two instances corresponding to distinct classes may be compatible,
meaning that a third instance exists satisfying the two sufficient reasons for contradictory classes. 
According to \cite{Amgoud23b}, it is also not possible to define an explanation function that generates 
subset-minimal sufficient reasons and guarantees existence of explanations (\textit{success}) and their global consistency. Consequently, in the same paper the authors introduced various functions that guarantee consistency at the cost of success while the functions defined in \cite{ecai23} ensure success and violate consistency. However, there is no explanation function in the literature  which generates 
dataset-based abductive explanations while guaranteeing both properties. 


Following an axiomatic approach, we provide in this paper a \textbf{complete} investigation of 
functions that generate feature-based explanations from datasets. 
We start by proposing formal properties of explainers, then identify all the families of 
explainers that may be defined, shed light on the pros/cons of each family and on the links 
between the families. The findings unravel the origins of the above issues and uncover the 
full landscape of sample-based explainers. 
More precisely, the contributions of the paper are sixfold: 
\vspace{-0.15cm}
\begin{itemize}
    \item We propose a set of axioms - desirable properties - that explainers may satisfy. 
          Three axioms are borrowed from the literature while others are new. 
    \item We delve into the relationships between the axioms and shed light on problematic incompatibilities 
          of some of them. 
          In particular, we show that the above-mentioned consistency property is incompatible with subset-minimality and success. 
    \item We fully characterize the entire family of explainers that generate weak abductive explanations. 
    \item We also characterize all its sub-families that satisfy subsets
          of compatible properties. 
    \item We introduce the \textbf{first} (broad family of) explainers that guarantee explanations and their global consistency.
    \item We discuss some of its \textbf{tractable} instances including the \textit{irrefutable explainer} and the \textit{surrogate explainer} which is based on surrogate models.
\end{itemize}
\vspace{-0.15cm}

The paper is organized as follows: Section~\ref{back} introduces sample-based explainers, 
a list of axioms, their links, and a characterization of the main family of explainers that 
satisfy two key axioms. 
Section~\ref{families} presents three sub-families of explainers that satisfy large subsets of compatible axioms,  
Section~\ref{sec:tractable} discusses three tractable examples of coherent explainers, and 
Section~\ref{secWE} illustrates them on a real-world example. 
Section~\ref{conclusion} discusses the obtained results, and 
the last section is devoted to related work and some concluding remarks. 
All proofs are given in the appendices.

\begin{table*}
	\begin{center}	
		\begin{tabular}{ll} 
			\hline\hline
   			\rowcolor{maroon!10} \textit{Feasibility} & $\forall \bQ = \langle \mbb{T}, \kappa, \D, x \rangle$, 
			$\forall E \in \bL(\bQ)$, $E \subseteq x$.  \\
			
			\rowcolor{maroon!10} \textit{Validity}  & $\forall \bQ = \langle \mbb{T}, \kappa, \D, x \rangle$, 
			$\forall E \in \bL(\bQ)$, $\nexists y \in \D$ s.t. $E \subseteq y$ and 
			$\kappa(y) \neq \kappa(x)$. \\
   
			\rowcolor{maroon!10}\textit{Success} & $\forall \bQ = \langle \mbb{T}, \kappa, \D, x \rangle$, $\bL(\bQ) \neq \emptyset$.  \\
			

			\rowcolor{maroon!10}\textit{Coherence}      & $\forall \bQ = \langle \mbb{T}, \kappa, \D, x \rangle$ and 
			$\bQ' = \langle \mbb{T}, \kappa, \D, x' \rangle$ s.t. $\kappa(x) \neq \kappa(x')$, \\ 
			\rowcolor{maroon!10} & then $\forall E \in \bL(\bQ)$, $\forall E' \in \bL(\bQ')$, 
			$E \cup E'$ is inconsistent. \\ 
   
   \textit{Irreducibility} & $\forall \bQ = \langle \mbb{T}, \kappa, \D, x \rangle$, $\forall E \in \bL(\bQ)$, $\forall l \in E$, 
			$\exists x' \in \D$ s.t. $\kappa(x') \neq \kappa(x)$ and $E \setminus \{ l \} \subseteq x'$. \\
 \textit{Strong Irreducibility} & $\forall \bQ = \langle \mbb{T}, \kappa, \D, x \rangle$, $\forall E \in \bL(\bQ)$, $\forall l \in E$, 
			$\exists x' \in \mbb{F}(\mbb{T})$ s.t. $\kappa(x') \neq \kappa(x)$ and $E \setminus \{ l \} \subseteq x'$. \\
			
			\textit{Completeness}  & $\forall \bQ = \langle \mbb{T}, \kappa, \D, x \rangle$, 
			$\forall E \subseteq x$, if $E \notin \bL(\bQ)$, then $\exists y \in \D$ s.t. $E \subseteq y$ and 
			$\kappa(y) \neq \kappa(x)$. \\

            \textit{Strong Completeness}  & $\forall \bQ = \langle \mbb{T}, \kappa, \D, x \rangle$, 
			$\forall E \subseteq x$, if $E \notin \bL(\bQ)$, then $\exists y \in \mbb{F}(\mbb{T})$ s.t. $E \subseteq y$ and 
			$\kappa(y) \neq \kappa(x)$. \\
			
			\textit{Monotonicity}    & $\forall \bQ = \langle \mbb{T}, \kappa, \D, x \rangle$,  
			$\forall \bQ' = \langle \mbb{T}, \kappa, \D', x\rangle$, if  $\D \subseteq \D'$, then $\bL(\bQ) \subseteq \bL(\bQ')$. \\
			
			\textit{Counter-Monotonicity (CM)} & $\forall \bQ = \langle \mbb{T}, \kappa, \D, x \rangle$,  
			$\forall \bQ' = \langle \mbb{T}, \kappa, \D', x\rangle$, if $\D \subseteq \D'$, then $\bL(\bQ') \subseteq \bL(\bQ)$.   \\

			
%
			\hline\hline
		\end{tabular}
		\caption{Formal Properties of Explainer $\bL$.}
		\label{tab1}
	\end{center}	
\end{table*}
\section{Sample-based Explainers}\label{back}

Throughout the paper, we consider a \textit{classification theory} as a tuple $\mbb{T} = \langle \F, \d, \C\rangle$ 
comprising a finite set $\F$ of \textit{features}, a function $\d$ which returns the \textit{domain} of every feature,  
where $\d(.)$ is finite and $|\d(.)| > 1$, and a finite set $\C$ of \textit{classes} with $|\C| \geq 2$.  
A \emph{literal} in $\mbb{T}$ is a pair $(f,v)$ where $f \in \F$ and $v \in \d(f)$. 
A set $L$ of literals is \emph{consistent} iff, for any two elements $(f,v)$ and $(f',v')$ of $L$, if $f = f'$, 
then $v = v'$. 
A \textit{partial assignment} is any set of literals with each feature in $\F$ occurring at most once; 
it is called an \textit{instance} when every feature appears once. Notice that partial assignments and 
instances are consistent. 
We denote by $\mbb{E}(\mbb{T})$ the set of all possible partial assignments and by 
$\mbb{F}(\mbb{T})$ the \textit{feature space}, i.e., the set of all instances, of theory $\mbb{T}$.  
We consider a \textit{classifier} on a theory $\mbb{T}$ as a function 
$\kappa : \mbb{F}(\mbb{T}) \rightarrow \C$, i.e., mapping every instance in $\mbb{F}[\mbb{T}]$ to a class in the set $\C$ of classes. 


In this paper, we are interested in explaining decisions taken by a classifier $\kappa$ for instances of a theory $\mbb{T}$. Explanations are generated using a \textbf{subset} of the feature space. For the sake of generality, its origin is left unspecified and it may be any sample of instances.

\begin{definition}\label{questions}
	A \emph{question} is a tuple $\bQ = \langle \mbb{T}, \kappa, \D, x \rangle$ such that $\mbb{T}$ is a 
	classification theory, $\kappa$ is a classifier on $\mbb{T}$,  $\D \subseteq \mbb{F}(\mbb{T})$, and
	$x \in \D$. 
\end{definition}

An \textit{explainer} is a function that takes as input a question and outputs a set 
of explanations. In the paper, we focus on \textit{feature-based explanations} which describe 
the input features that contribute to a classifier’s output for a given instance. 
Such explanations are thus partial assignments of the theory.

\begin{definition}\label{explainer}
An \emph{explainer} is a function $\bL$ mapping every question 
$\bQ = \langle \mbb{T}, \kappa, \D,x \rangle$ into  a subset of $\mbb{E}(\mbb{T})$. 
Every $E \in \bL(\bQ)$ is an \emph{explanation} of $\kappa(x)$ under the dataset $\D$.
\end{definition}

To theoretically scrutinize such explainers, we propose axioms (or formal properties) that they may satisfy. 
Axioms are important not only for a better understanding of the explanation process in general, but also for 
clarifying the basic assumptions underlying explainers, highlighting pros and cons of an explainer, 
comparing different (families of) explainers, and 
for also identifying families of explainers that have not been explored yet. 
We provide in Table~\ref{tab1} a list of ten, of which the three first ones (Success, Strong Irreducibility, Coherence) have counterparts in \cite{ecai23} while the others are \textbf{new}. 
\textit{Success} guarantees at least one explanation to every question. 
\textit{(Strong) Irreducibility} ensures that explanations do not contain unnecessary information to the explained decision. Irreducibility is tailored to sample-based explainers while its strong 
version concerns explainers that use the feature space. 
\textit{Coherence} is crucial for sample-based explainers as it ensures global consistency of all generated 
explanations. Technically, two explanations for two instances corresponding to distinct classes should not be compatible, 
otherwise there would exist a third instance satisfying the two sufficient reasons for contradictory classes. Let us now introduce the new axioms. 
As we are interested in feature-based explanations, \textit{ Feasibility} states that an explanation should be part of the instance being explained. 
\textit{Validity} ensures that generated explanations are locally consistent within the dataset. \textit{(Strong) Completeness} ensures that no valid explanations are omitted.
\textit{Monotonicity} states that the process of constructing explanations is monotonic, that is, an explanation remains valid if the dataset is enlarged. 
\textit{Counter-monotonicity} (CM) states that enlarging a dataset can only lead to discarding explanations. 

%

\vspace{0.1cm}

We show that some axioms follow from others. Despite these dependencies, we keep 
all the axioms because some explainers may satisfy a property but violate some of those from which 
it follows. Hence, they are useful for discriminating explainers.

\begin{proposition}\label{links}
The following implications hold. 
\vspace{-0.15cm}
\begin{itemize}
\item Completeness $\Rightarrow$ Strong Completeness, \\ 
Strong Completeness $\Rightarrow$ Success.
\item Irreducibility $\Rightarrow$ Strong Irreducibility.
\item Success, Feasibility and Coherence $\Rightarrow$ Validity.
\item Feasibility, Validity and Completeness $\Rightarrow$ Counter-Monotonicity. 
\end{itemize} 
\end{proposition}

Despite the importance of all the axioms, some of them are incompatible, 
i.e., they cannot be satisfied all together by an explainer. 

\begin{theorem}\label{th:incompatibility}
The axioms of every set $\mathbf{I}_{i=1,5}$ below are incompatible.

\vspace{-0.15cm}
\begin{description}
\item [$(\mathbf{I}_1)$] Feasibility, Success, Coherence and Irreducibility. 
\item [$(\mathbf{I}_2)$] Feasibility, Coherence and Completeness. 
\item [$(\mathbf{I}_3)$] Strong Irreducibility and Strong Completeness. 
\item [$(\mathbf{I}_4)$] Feasibility, Validity, Success, Irreducibility and Monotonicity.  
\item [$(\mathbf{I}_5)$] Feasibility, Validity, Success, Irreducibility and CM. 
\end{description}	
\end{theorem}

Note that the important property of Irreducibility is incompatible with most of key axioms 
like Success and Coherence. This shows that subset-minimality of explanations has a cost in 
the sample-based explaining context. Coherence (or global consistency of explanations) 
cannot be ensured together with some other properties. 
Below is a list of five maximal sets of compatible axioms. 

\begin{theorem}\label{th:compatibility}
The axioms of every set $\mathbf{C}_{i=1,5}$ below are compatible.

\vspace{-0.15cm}
\begin{description}
\item [$(\mathbf{C}_1)$] Feasibility, Validity, Success, (Strong) Completeness, CM. 
\item [$(\mathbf{C}_2)$] Feasibility, Validity, Success, (Strong) Irreducibility. 
\item [$(\mathbf{C}_3)$] Feasibility, Validity, Success, Coherence, Monotonicity, CM, Strong Completeness.  
\item [$(\mathbf{C}_4)$] Feasibility, Validity, Success, Coherence, Monotonicity, CM,  
                         Strong Irreducibility.                         
\item [$(\mathbf{C}_5)$] Feasibility, Validity, Coherence, (Strong) Irreducibility,
Monotonicity, CM. 
\end{description}	
\end{theorem}


From the above results, we have the following characterisation.

\begin{theorem}
    Considering all sets of axioms among those listed in Table~\ref{tab1}
    which include Feasibility and Validity, the only minimal incompatible sets  
    are $\mathbf{I}_1,\ldots,\mathbf{I}_5$ and the only maximal
    compatible sets are $\mathbf{C}_1,\ldots,\mathbf{C}_5$.
\end{theorem}

Recall that Feasibility defines the \textbf{(feature-based) type} of explanations and 
Validity guarantees their \textbf{local consistency} within the datatset. They are thus 
mandatory in our investigation of feature-based explanations. In what follows, we provide a full characterization of 
the entire family of explainers that satisfy them. We show that instances of this family 
generate the so-called 
\textit{weak abductive explanations} \cite{ecai23,ICML21}. 
Such explanations are \textit{sufficient reasons} for predicting the classes of instances of 
a classification theory. They are also known as 
\textit{Prime Implicants} in \cite{AudemardBLM23,darwiche20,IgnatievNM19,ICML21}. 
In this literature, they are generated from the \textbf{whole feature space}, answering thus only the subset of questions where $\D = \mbb{F}(\mbb{T})$. In what follows, we present their 
counterpart in a sample-based setting.

\begin{definition}\label{dwaxp}
Let $\bQ = \langle \mbb{T}, \kappa,\D, x \rangle$ be a question and $E \in \mbb{E}(\mbb{T})$. 
$E$ is a \emph{weak abductive explanation} (dwAXp) of $\kappa(x)$ iff:
\begin{itemize}
\item $E \subseteq x$, 
\item $\forall y \in \D$, if $E \subseteq y$ then $\kappa(y) = \kappa(x)$.
\end{itemize} 
We denote by $\bLdw$ the explainer generating dwAXp's and by $\bLw$ the explainer 
that generates dwAXp's from the feature space $\mbb{F}(\mbb{T})$.
\end{definition}

Note that $\bLw$ returns for any question 
$\bQ = \langle \mbb{T}, \kappa,\D, x \rangle$ the dwAXp of 
$\bQ' = \langle \mbb{T}, \kappa,\mbb{F}(\mbb{T}), x \rangle$ (i.e., $\bLw(\bQ) = \bLdw(\bQ')$), 
thus neglecting $\D$ and using the feature space instead. 

\begin{example}\label{ex2}
Consider the theory $\mbb{T}_1$ made of two binary features and a binary classifier 
$\kappa_1$ that gives the predictions in the table below. 
Consider also the dataset $\D_1 = \{x_1,x_2\}$ 
and let us focus on the question $\bQ_1 = \langle\mbb{T}_1, \kappa_1, \D_1, x_1\rangle$. 
It can be checked that 
$\bLw(\bQ_1) = \{E_3\}$ while 
$\bLdw(\bQ_1) = \{E_2, E_3\}$ ($E_2$ is not valid on $\mbb{F}(\mbb{T}_1)$).  
\vspace{-0.3cm}
\begin{multicols}{2}
\begin{center}
	\begin{tabular}{c|cc|c}\hline
	       $\mbb{F}(\mbb{T}_1)$   &$f_1$& $f_2$ & $\kappa_1(x_i)$   \\\hline
\rowcolor{maroon!10}	$x_1$    & 0   & 0     & 0 \\
				          $x_2$    & 0   & 1     & 1 \\ \hline            
				        $x_3$    & 1   & 0     & 1 \\
				          $x_4$    & 1   & 1     & 0 \\\hline
	\end{tabular}
\end{center}
\qquad  
\begin{itemize}
	\item [] \item [] 
	\item [] \qquad $E_1 = \{(f_1,0)\}$ 
	\item [] \qquad $E_2 = \{(f_2,0)\}$
	\item [] \qquad $E_3 = \{(f_1,0),(f_2,0)\}$ 
\end{itemize}	
\end{multicols}
\end{example}

Below we provide a characterization of the \textbf{entire family} of explainers that provide weak abductive explanations (dwAXp). We show that they are the \textbf{only ones}  to satisfy Feasibility and Validity. 

\begin{theorem}\label{th-cara}
An explainer $\bL$ satisfies Feasibility and Validity \textbf{iff} for any question 
$\bQ = \langle \mbb{T}, \kappa, \D, x \rangle$, $\bL(\bQ) \subseteq \bLdw(\bQ)$.	
\end{theorem}

From a computational complexity stance, the problem of testing 
whether an explanation belongs to $\bLw$ is
intractable. Indeed, it has been shown in \cite{CP21} that the complexity of testing 
whether $E$ is a weak abductive explanation in the whole feature space is co-NP-complete.
However, the problem is tractable in the sample setting (i.e., 
using the function $\bLdw$). Indeed, it has been shown in \cite{ecai23} that testing the 
validity of sample-based weak abductive explanations is linear in the size of the sample whatever the classifier $\kappa$. 

\begin{property}\cite{ecai23}
	Let $\mbb{T} = \langle \F, \d, \C\rangle$ be a theory with $n = |\F|$, $\D \subseteq \mbb{F}(\mbb{T})$ 
	with $m = |\D|$ and $E \in \mbb{E}(\mbb{T})$. 
		Testing whether $E$ is a dwAXp can be achieved in $O(mn)$ time.
\end{property}
\section{Sub-Families of Explainers}\label{families} 

In this section, we scrutinize the types of explainers that can be defined from the set of axioms. 
More precisely, we characterize all sub-families of explainers which select in different ways 
subsets of weak abductive explanations. Each sub-family satisfies one of the first four 
subsets $(\mathbf{C}_i)$ of compatible axioms identified in Theorem~\ref{th:compatibility}. 
The last set $\mathbf{C}_5$ is satisfied by the trivial explainer which 
returns the empty set for each question. Furthermore, existing explainers from \cite{Amgoud23b} satisfy Coherence and Irreducibility at the cost of Success.
\subsection{Weak Abductive Explainers} 

The first sub-family contains two explainers, namely the function $\bLdw$ which returns \textbf{all}  weak abductive explanations and $\bLw$ which explores all the feature space. We show that $\bLdw$ privileges 
\textbf{Completeness} in the two conflicts $(\mathbf{I}_2)$ and $(\mathbf{I}_3)$. Furthermore, 
it is the \textbf{only} explainer that satisfies 
Feasibility, Validity and Completeness.

\begin{theorem}\label{th:dw}
An explainer $\bL$ satisfies Feasibility, Validity and Completeness \textbf{iff} $\bL = \bLdw$.
\end{theorem}

We summarize below the complete list of axioms satisfied and violated respectively by the function 
$\bLdw$. Notice that $\bLdw$ satisfies the first largest set $(\mathbf{C}_1)$ of compatible axioms 
from Theorem~\ref{th:compatibility}.

\begin{theorem}\label{th:bldw}
The explainer $\bLdw$ satisfies Feasibility, Validity, Success, (Strong) Completeness and CM. 
It violates the remaining axioms.
\end{theorem}

The explainer $\bLdw$ satisfies thus all the axioms that are compatible with Completeness except Monotonicity. The latter is violated due to incompleteness of information of a dataset. If the latter is enlarged, it is possible that a dwAXp will no longer be valid as the new dataset would contain a new instance which invalidates the explanation (violation of the second condition of Def.~\ref{dwaxp}). 

\begin{exs}{ex2}
Consider the dataset $\D_2 = \{x_1,x_2, x_3\}$. 
Note that $E_2 \notin \bLdw(\bQ_2)$, where $\bQ_2 = \langle\mbb{T}_1, \kappa_1, \D_2, x_1\rangle$. 
\end{exs}

The second function $\bLw$ which generates weak abductive explanation from the feature space is the \textbf{only} explainer that satisfies the four axioms  Feasibility, Validity, Monotonicity and 
Strong Completeness.

\begin{theorem}\label{th:w}
An explainer $\bL$ satisfies Feasibility, Validity, Monotonicity and Strong Completeness  \textbf{iff} $\bL = \bLw$.
\end{theorem}

Unlike $\bLdw$, the explainer $\bLw$ satisfies Coherence and 
all the other axioms of the set $\mathbf{C_3}$.

\begin{theorem}\label{th:blw}
The explainer $\bLw$ satisfies Feasibility, Validity, Success, Coherence, Monotonicity, CM and 
Strong  Completeness. It violates (Strong) Irreducibility and Completeness.
\end{theorem}

The violation of Completeness by $\bLw$ is due to the fact that the axiom is  
tailored for sample-based explainers. 
It is also worth noticing that since $\bLdw$ satisfies Counter-Monotonicity, it recovers any explanation that is generated by $\bLw$ from the entire feature space of the theory at hand, i.e., it finds all the ``real''  explanations under \textbf{any} dataset (see the Fidelity property in Table~\ref{tab2}). 


However, the function $\bLdw$ may generate additional explanations which are valid only in the dataset. These new explanations are the main culprits for the violation of Coherence by $\bLdw$.  

\subsection{Concise Abductive Explainers} 

Our next characterization concerns explainers that satisfy \textbf{Irreducibility} in the four conflicts  
($\mathbf{I}_1$, $\mathbf{I}_3$, $\mathbf{I}_4$, $\mathbf{I}_5$). We show that they generate subset-minimal weak abductive explanations. Note that such explainers have been studied in \cite{Amgoud23a,ecai23}.

\begin{definition}\label{daxp}
Let $\bQ = \langle \mbb{T}, \kappa,\D, x \rangle$ be a question and $E \in \mbb{E}(\mbb{T})$. 
$E$ is a \emph{concise abductive explanation} (cAXp) of $\kappa(x)$ iff:
\begin{itemize}
	\item $E \in \bLdw(\bQ)$, 
	\item $\nexists E' \in \bLdw(\bQ)$ such that $E' \subset E$.
\end{itemize}		
We denote by $\bLda$ the explainer generating cAXp. 
\end{definition}

\begin{exs}{ex2}
$\bLda(\bQ_1) = \{E_2\}$ and $\bLdw(\bQ_1) = \{E_2, E_3\}$.  
\end{exs}	

%
%
%

We show next that explainers that provide cAXp's are the \textbf{only ones} satisfying  
Feasibility, Validity, and Irreducibility. 

\begin{theorem}\label{th:da}
An explainer $\bL$ satisfies Feasibility, Validity, and Irreducibility iff for any question 
$\bQ = \langle \mbb{T}, \kappa, \D, x \rangle$, $\bL(\bQ) \subseteq \bLda(\bQ)$.
\end{theorem}

An instance of this family of explainers is ${\bLda}$ which generates \textbf{all} the subset-minimal weak abductive explanations. We show that it satisfies the second set $(\mathbf{C}_2)$ of compatible axioms, and thus violates several properties due to their incompatibility with Irreducibility.

\begin{theorem}\label{th:blda}
The explainer $\bLda$ satisfies Feasibility, Validity, Success and (Strong) Irreducibility. 
It violates the remaining axioms.
\end{theorem}

 \begin{property}\label{p-da-dw}
 	Let $\bQ = \langle \mbb{T}, \kappa, \D, x \rangle$ be a question. 
 	\begin{itemize}
 		\item $x \in \bLdw(\bQ)$.
 		\item $\bLda(\bQ) \subseteq \bLdw(\bQ)$.
 	\end{itemize}	
 \end{property}	
 

Another explainer which generates concise abductive explanations is $\bLa$ which explores the feature space. Such explanations, known as prime implicants, have been largely studied in the literature (e.g., \cite{AudemardBBKLM22,marquis22a,Audemard23,CP21,darwiche20,marquis22b,IgnatievNM19,joao19b,darwiche18}).

\begin{definition}\label{axp}
A \emph{prime implicant explainer} is a function $\bLa$ such that for any question 
$\bQ = \langle \mbb{T}, \kappa,\D, x \rangle$, $\bLa(\bQ) = \bLda(\bQ')$, 
where $\bQ' = \langle \mbb{T}, \kappa,\mbb{F}(\mbb{T}), x \rangle$. 
\end{definition}

The function $\bLa$ satisfies all the axioms of the set $(\mathbf{C}_4)$, hence all axioms  
except (Strong) Completeness, and also Irreducibility despite the fact it generates subset-minimal 
explanations. This violation is merely due to the definitions of the axiom which is sample-based.

\begin{theorem}\label{th:bla}
The function $\bLa$ satisfies all the axioms except Irreducibility and (Strong) Completeness. 
\end{theorem}

\noindent \textbf{Remark:} Note that the inclusion $\bLa(\bQ) \subseteq \bLda(\bQ)$ may not hold. 

\begin{exs}{ex2}
$\bLda(\bQ_1) = \{E_2\}$ while $\bLa(\bQ_1) = \{E_3\}$.   
\end{exs}	

From a computational complexity perspective, the problems of \textit{testing} and 
\textit{finding} one explanation of $\bLa$ are intractable. Indeed, it has been shown in  
\cite{COOPER2023} that the complexity of testing whether a partial assignment $E$ is 
an explanation is in $\text{P}^{\text{NP}}$ (and is NP-hard), and the complexity of finding one explanation is  
in $\text{FP}^{\text{NP}}$. However, the two problems are tractable when $\bLda$ is 
used  \cite{ecai23}.

\begin{property}\cite{ecai23}
Let $\mbb{T} = \langle \F, \d, \C\rangle$ be a theory with $n = |\F|$, $\D \subseteq \mbb{F}(\mbb{T})$ with $m = |\D|$ and $E \in \mbb{E}(\mbb{T})$. 
Finding a cAXp can be achieved in $O(mn^2)$ time.
\end{property}
\subsection{Coherent Sample-based Abductive Explainers}\label{coh:exp}

The two previous sub-families of explainers, in particular their sample-based instances ($\bLdw, \bLda$) privilege Completeness 
and Irreducibility at the cost of Coherence, hence their explanations may be globally inconsistent. 
Indeed, we have seen that Coherence is incompatible with completeness, and with the pair (Success, Irreducibility). In what follows, we propose the \textbf{first (family of) explainers} that satisfy 
Success and Coherence. We start by  showing that its instances generate weak abductive explanations, 
meaning that they select a collection of globally consistent dwAXp's.

\begin{theorem}\label{coh+success}
If an explainer $\bL$ satisfies Feasibility, Success and Coherence, then for any question $\bQ$, $\bL(\bQ) \subseteq \bLdw(\bQ)$.
\end{theorem}

Before introducing the novel family, let us first provide some useful notions, including a \textit{coherent set} of partial assignments. The latter is nothing more than a set which satisfies the consistency condition of the Coherence axiom. 

\begin{definition}\label{coherent-set}
	Let $\mbb{T}$ be a theory,  $X \subseteq \mbb{E}(\mbb{T})$, $\D \subseteq \mbb{F}(\mbb{T})$ and $\kappa$ a classifier. $X$ is \emph{coherent} under $(\D, \kappa)$ iff $\nexists E, E'  \in X$ such that: 
	\begin{itemize}
		\item $E \cup E'$ is consistent, and 
		\item $\exists y, z \in \D$ such that $E \subseteq y$, $E' \subseteq z$, and $\kappa(y) \neq \kappa(z)$. 
	\end{itemize}    
	Otherwise, $X$ is said to be \emph{incoherent}. 
\end{definition}

\begin{exs}{ex2}
Let $E_1 =\{(f_1,0)\}$ and $E_2 = \{(f_2,0)\}$. The set $\{E_1, E_2\}$ is incoherent since 
$E_1 \cup E_2$ is consistent, $E_1 \subseteq x_2$, $E_2 \subseteq x_1$ and 
$\kappa_1(x_2) \neq \kappa_1(x_1)$.
\end{exs}

Since no two distinct instances of a feature space are consistent with each other, 
every dataset is coherent. 

\begin{property}\label{sample-coh}
	For any $\D \subseteq \mbb{F}(\mbb{T})$ and any classifier $\kappa$, 
	$\D$ is coherent under ($\D,\kappa$). 	
\end{property}

Let us introduce the notion of \textit{envelope}, which is a coherent set of 
weak abductive explanations covering every instance of a dataset. 

\begin{definition}\label{def:envelope}
	Let $\mbb{T}$ be a theory, $\D \subseteq \mbb{F}(\mbb{T})$, $\kappa$ a classifier. 
	An \emph{envelope} under $(\D, \kappa)$ is any $X \subseteq \mbb{E}(\mbb{T})$ such that the following hold:
	\begin{itemize}
		\item $X$ is coherent under  $(\D, \kappa)$, 
		\item $\forall E \in X$, $\exists x \in \D$ such that $E \in \bLdw(\langle\mbb{T},\kappa,\D, x\rangle)$, 
		\item $\forall x \in \D$, $\exists E \in X$ such that $E \subseteq x$.   
	\end{itemize}	
	Let $\mathtt{Coh}(\D,\kappa)$ be the set of all envelopes under $(\D, \kappa)$.  
\end{definition}

It is easy to see that the set $\mathtt{Coh}(\D,\kappa)$ is not empty as it contains at least the dataset $\D$ 
itself. Furthermore, each $X \in \mathtt{Coh}(\D,\kappa)$ contains a subset of the dwAXp's of every instance in $\D$.

\begin{proposition}\label{prop:coh}
	Let $\mbb{T}$ be a theory, $\D \subseteq \mbb{F}(\mbb{T})$, and $\kappa$ a classifier. 
	\begin{itemize} 
		\item Let $X \in \mathtt{Coh}(\D,\kappa)$. For any $x \in \D$, for any $E \in X$, 
		if $E \subseteq x$, then $E \in \bLdw(\langle\mbb{T},\kappa,\D, x\rangle)$.
		\item $\D \in \mathtt{Coh}(\D,\kappa)$. 
        \item For any $E \in \bLdw(\langle\mbb{T},\kappa,\D, x\rangle)$, 
        $\exists X \in \mathtt{Coh}(\D,\kappa)$ such that $X$ is a subset-maximal envelope of $\mathtt{Coh}(\D,\kappa)$ and $E \in X$. 
	\end{itemize}
\end{proposition}

\begin{exs}{ex2}
$\mathtt{Coh}(\D_1,\kappa_1)$ contains the following envelopes:
	\begin{itemize}
		\item $X_1 = \D_1 = \{x_1, x_2\}$, 
		\item $X_2 = \{\{(f_2,0)\},  \{(f_2,1)\}\}$, 
		\item $X_3 = \{\{(f_2,0)\},  x_2\}$, 
		\item $X_4 = \{\{(f_2,1)\},  x_1\}$, 
        \item $X_5 = \{x_1, x_2, \{(f_2,0)\}, \{(f_2,1)\}\}$. 
	\end{itemize}
	Note that, for example, the set $\{\{(f_2,0)\}\}$ is not an envelope since $\{(f_2,0)\}$ is not a dwAXp of the instance 
    $x_2$.
\end{exs}

Let us now define the novel family of \textit{coherent explainers} whose instances return, under any dataset, an envelope.

\begin{definition}\label{def:coh}
	A \emph{coherent explainer} is a function $\bLc$ such that for any $\bQ = \langle \mbb{T}, \kappa, \D, x \rangle$, the following hold:
	\begin{itemize}
		\item $\bLc(\bQ) = \{E{\in}X  \mid  E{\subseteq}x\}$, where 
		$X \! = \! \bigcup\limits_{z \in \D}\bLc(\langle \mbb{T},\kappa,\D,z \rangle)$, 
		\item $X \in \mathtt{Coh}(\D,\kappa)$. 
	\end{itemize} 
\end{definition}

We show next that coherent explainers are the \textbf{only ones} to satisfy together the three axioms Success, Feasibility and Coherence, hence to guarantee an explanation to every instance as well as global consistency of explanations.

\begin{theorem}\label{cara:coh}
An explainer $\bL$ satisfies Feasibility, Success and Coherence \textbf{iff} $\bL$ is a 
coherent explainer.
\end{theorem}

In addition to the three properties (Success, Feasibility, Coherence), any coherent explainer 
satisfies Validity but violates Irreducibility and Completeness.

\begin{theorem}\label{cara:coh2}
Any coherent explainer satisfies Feasibility, Validity, Success, Coherence. 
It violates Irreducibility and Completeness.
\end{theorem}

The sub-family of coherent explainers is very \textbf{broad} and encompasses various explainers including the two functions $\bLw$ and $\bLa$. Some of its instances are \textbf{tractable} as we will see in the next section. 
\section{Tractable Coherent Explainers}\label{sec:tractable}

The family of coherent explainers is large and covers a variety of functions because 
from a given dataset $\D$, one may generate several envelopes, 
i.e., $|\mathtt{Coh}(\D,\kappa)| \geq 1$. 
In this section, we discuss three functions 
whose explanations can be generated in polynomial time.

\subsection{Trivial Explanation Function}\label{sec:trivial}

The first function, called \emph{trivial explainer}, is the one whose envelope is the dataset itself. 
We have seen in Proposition~\ref{prop:coh} that any dataset is an envelope. This function assigns thus 
a single explanation to every question, which is nothing more than the instance being explained.

\begin{definition}\label{trivial-exp}
A \emph{trivial} explainer is a function $\bLt$ such that for any question 
$\bQ = \langle \mbb{T}, \kappa, \D, x \rangle$,  
$\bLt(\bQ) = \{x\}$.    
\end{definition}

The function $\bLt$ is clearly a coherent explainer. 

\begin{property}
    The function $\bLt$ is a coherent explainer. 
\end{property}


\begin{theorem}\label{th:blt}
The function $\bLt$ satisfies all the axioms except (Strong) Irreducibility and (Strong) Completeness. 
\end{theorem}
    
Obviously, an explainer which returns an instance as explanation of its outcome is not informative 
and definitely not useful for users. 

\subsection{Irrefutable Explanation Functions}\label{sec:irrefutable}

We now introduce the novel \textit{irrefutable explainer}. It uses 
the so-called \textit{irrefutable envelope} that contains \textit{all} non-conflicting weak abductive explanations. 

\begin{definition}\label{irr-envelope}
Let $\mbb{T}$ be a theory, $\D \subseteq \mbb{F}(\mbb{T})$, $\kappa$ a classifier, and 
$\mbb{W} = \bigcup\limits_{x \in \D}\bLdw(\langle \mbb{T}, \kappa, \D, x \rangle)$. 
An \emph{irrefutable envelope} under $(\D, \kappa)$ is $X \subseteq \mbb{E}(\mbb{T})$ such that the following hold:
\begin{itemize}
	\item $X \subseteq \mbb{W}$, 
	\item $\forall E \in X$, $\nexists E' \in \mbb{W}$ such that $\{E, E'\}$ is incoherent, 
	\item $\nexists X' \supset X$ that satisfies the above conditions.   
\end{itemize}	
\end{definition}

The next result shows that the irrefutable envelope is unique. 

\begin{property}\label{uniqueness}
Let $\mbb{T}$ be a theory, $\D \subseteq \mbb{F}(\mbb{T})$, and $\kappa$ a classifier. 
For all $X , X' \subseteq \mbb{E}(\mbb{T})$, if $X$ and $X'$ are irrefutable envelopes under $(\D,\kappa)$, then $X = X'$.
\end{property}

\noindent \textbf{Notation:} Throughout the paper, $\mathtt{Irr}(\D,\kappa)$ denotes the irrefutable 
envelope under $(\D, \kappa)$ in theory  $\mbb{T}$.  


An irrefutable envelope contains the whole dataset under which it is defined. Furthermore, it is 
the intersection of all subset-maximal envelopes.

\begin{proposition}\label{prop:irr-env}
Let $\mbb{T}$ be a theory, $\D \subseteq \mbb{F}(\mbb{T})$, and $\kappa$ a classifier. 
\begin{itemize}
	\item $\D \subseteq \mathtt{Irr}(\D,\kappa)$.   
    \item $\mathtt{Irr}(\D,\kappa) \in \mathtt{Coh}(\D,\kappa)$.
    \item $\mathtt{Irr}(\D,\kappa) = \bigcap\limits_{X_i \in \mathcal{S}} X_i$, 
    where $\mathcal{S}$ is the set of all subset-maximal envelopes of $\mathtt{Coh}(\D,\kappa)$.
\end{itemize}
\end{proposition}

We are now ready to present the instance of the family of coherent explainers that is based on irrefutable envelopes. 

\begin{definition}\label{irrefutable}
	An \emph{irrefutable explainer}  is a function $\bLr$ such that for any question 
	$\bQ = \langle \mbb{T}, \kappa, \D, x \rangle$,  
	$$\bLr(\bQ) = \{E \in \mathtt{Irr}(\D,\kappa) \ \mid \ E \subseteq x\}.$$   
\end{definition}

\begin{exs}{ex2}
Recall that $\mathtt{Coh}(\D_1,\kappa_1) = \{X_i \ \mid i= 1,\ldots,5\}$. 
$\mathtt{Irr}(\D_1,\kappa_1) = X_5$, hence 
$\bLr(\bQ_1) = \{x_1, \{(f_2,0)\}\}$.
\end{exs}

\begin{example}\label{ex3}
	Consider the theory $\mbb{T}_2$ made of two binary features and a binary classifier 
    $\kappa_2$ which provides the predictions below for the three instances in $\D_2$. 
    Let $\bQ_i = \langle \mbb{T}_2, \kappa_2, \D_2, x_i \rangle$, with $i \in \{1,2,3\}$.
	\vspace{-0.3cm}
	\begin{multicols}{2}
		\begin{center}
			\begin{tabular}{c|cc|c}\hline
				              $\D_2$  &$f_1$  & $f_2$ & $\kappa_2(x_i)$   \\\hline
    \rowcolor{maroon!10}	$x_1$ & 0     & 0     &  0 \\
	            			$x_2$ & 1     & 0     &  0 \\
				            $x_3$ & 1     & 1     &  1 \\\hline
			\end{tabular}
		\end{center}
		\begin{itemize}
			\item [] $E_1 = \{(f_1,0)\}$ 
			\item [] $E_2 = \{(f_2,0)\}$ 
			\item [] $E_3 = \{(f_2,1)\}$ 
			\item [] $E_4 = x_1$ $\ E_5=x_2$  $E_6 = x_3$
		\end{itemize}	
	\end{multicols}
	\noindent Note that 
	$\bigcup\limits_{i= 1}^{3}\bLdw(\bQ_i) = \{E_1, \ldots, E_6\}$. 
	Note also that the set $\{E_1, E_3\}$ is incoherent and 
	$\mathtt{Irr}(\D_2,\kappa_2) = \{E_2, E_4, E_5, E_6\}$. So 
	$\bLr(\bQ_1) = \{E_2, E_4\}$, $\bLr(\bQ_2) = \{E_2,E_5\}$, $\bLr(\bQ_3) = \{E_6\}$.
\end{example}


\begin{property}\label{prop:lir}
        $\bLr$ is a coherent explainer, and 
  for any question $\bQ = \langle \mbb{T}, \kappa, \D, x \rangle$, 
            $x \in \bLr(\bQ)$.
\end{property}

Unlike the trivial explainer, $\bLr$ violates Monotonicity and Counter-Monotonicity.  

\begin{theorem}\label{th3}
The function $\bLr$ satisfies Feasibility, Validity, Success and Coherence. It violates all 
the other axioms.
\end{theorem}

\noindent \textbf{Remark:} Unlike $\bLdw$, the coherent explainers 
$\bLt$ and $\bLr$ do not necessarily recover the explanations that are generated from the feature space, i..e, the following inclusion may not hold: $\bL(\bQ') \subseteq {\bL}(\bQ)$, for 
$\bQ = \langle \mbb{T}, \kappa, \D, x \rangle$ and 
$\bQ' = \langle \mbb{T}, \kappa, \mbb{F}(\mbb{T}), x \rangle$. See the \textit{Fidelity} property in 
Table~\ref{tab2}.

\vspace{0.1cm}

In addition to guaranteeing the existence of explanations and their global consistency, the
irrefutable explainer is tractable.
Indeed, a subset-minimal explanation can be generated in polynomial time by a greedy algorithm~\cite{ChenT95} based
on the following theorem. 

\begin{theorem} \label{thm:irrpoly}
Let $\bQ = \langle\mbb{T},\kappa,\D,x\rangle$ be a question,
where $\kappa$ can be evaluated in polynomial time.
Testing whether $E \in \bLr(\bQ)$ can be achieved in polynomial time.
\end{theorem}

\begin{table*}
	\centering
	\begin{tabular}{l!{\vline\vline}c|c!{\vline\vline}c|c!{\vline\vline}c|c|c|c}
		\hline\hline
		Axioms / Explainers  & $\bLw$  &  $\bLa$  & $\bLdw$  &  $\bLda$  &  $\bLc$    & $\bLt$     &  $\bLr$    &  $\bLsu$ \\
		\hline\hline
  		Feasibility          & \cellcolor{maroon!10} $\checkmark$ & \cellcolor{maroon!10} $\checkmark$ & $\checkmark$ & $\checkmark$ & \cellcolor{maroon!10} $\checkmark$ &  \cellcolor{maroon!10} $\checkmark$ & \cellcolor{maroon!10} $\checkmark$  &  \cellcolor{maroon!10} $\checkmark$\\
		\hline
		Validity             & \cellcolor{maroon!10} $\checkmark$ & \cellcolor{maroon!10} $\checkmark$ & $\checkmark$ & $\checkmark$ & \cellcolor{maroon!10} $\checkmark$  & \cellcolor{maroon!10} $\checkmark$  & \cellcolor{maroon!10} $\checkmark$  &  \cellcolor{maroon!10} $\checkmark$  \\
		\hline	
	    Success              & \cellcolor{maroon!10}$\checkmark$ & \cellcolor{maroon!10} $\checkmark$ & $\checkmark$ & $\checkmark$ &  \cellcolor{maroon!10} $\checkmark$ & \cellcolor{maroon!10} $\checkmark$ & \cellcolor{maroon!10} $\checkmark$  & \cellcolor{maroon!10} $\checkmark$\\	
		\hline
          Coherence            & \cellcolor{maroon!10} $\checkmark$ & \cellcolor{maroon!10} $\checkmark$ & $\times$  & $\times$  & \cellcolor{maroon!10} $\checkmark$  & \cellcolor{maroon!10} $\checkmark$  & \cellcolor{maroon!10} $\checkmark$  &  \cellcolor{maroon!10} $\checkmark$  \\
		\hline	
		Irreducibility       & $\times$ & $\times$ & $\times$  & $\checkmark$ & $\times$   & $\times$   & $\times$   &  $\times$ \\ \hline
    Strong Irreducibility      & $\times$ & $\checkmark$ & $\times$  & $\checkmark$ & $-$   & $\times$   & $\times$   &  $\times$ \\
		\hline
		Completeness         & $\times$ & $\times$ & $\checkmark$ & $\times$  & $\times$   & $\times$   & $\times$   &  $\times$\\
		\hline
        Strong Completeness    & $\checkmark$ & $\times$ & $\checkmark$ & $\times$  & $-$   & $\times$   & $\times$   &  $\times$\\
		\hline
		Monotonicity         & $\checkmark$ & $\checkmark$ & $\times$  & $\times$  &   $-$      & $\checkmark$  & $\times$   &  $\times$ \\
		\hline
		Counter-Monotonicity    & $\checkmark$ & $\checkmark$ & $\checkmark$ & $\times$ &   $-$      & $\checkmark$  & $\times$   &  $\times$ \\
		\hline\hline
\rowcolor{maroon!30}	Fidelity: $\ $	$\bLw(\bQ) \subseteq \bL(\bQ)$ & $\checkmark$ & $\times$ &   $\checkmark$  & $\times$ & $-$ &$\times$ & $\times$ & $\times$ \\
		\hline\hline
	\end{tabular}
	\caption{The symbols $\checkmark$, $\times$ and $-$ stand for the axiom is satisfied, violated 
    and unknown respectively by the explainer.\\\\ $\ $ \\ $\ $} \vspace{0.5cm}
	\label{tab2}
\end{table*}
\subsection{Surrogate Classifier}\label{sec:surr}

In this section we describe another type of coherent explainers. 
Starting from a result in \cite{Amgoud22} stating that Coherence 
is ensured by explainers that provide sufficient reasons from the whole feature space, 
the idea is, given a question $\bQ = \langle\mbb{T},\kappa,\D,x\rangle$, to find a \textit{surrogate 
classifier} $\sigma$ which is equal to $\kappa$ on $\D$ but which allows tractable explaining 
on the whole feature space. Explanations of $\bQ$ are then explanations of    
$\bQ' = \langle\mbb{T},\sigma,\mathbb{F}(\mbb{T}),x\rangle$. 
Tractability follows if we choose $\sigma$ from a family of classifiers that allows 
tractable explaining on the whole feature 
space~\cite{AudemardKM20,CarbonnelC023,HuangIICA022,IzzaIM22}. 


The approach relies heavily on $\sigma$, thus the question of its \textbf{existence} naturally arises. The answer is fortunately positive. Indeed, it is always possible to find a \textit{decision tree} $\sigma$ 
for any model $\kappa$ with 100\% accuracy on any dataset $\D$ (by over-fitting,
if necessary). 
Each instance $x \in \D$ corresponds to a unique path from the root to a leaf of the 
decision tree and will be consulted, during construction
of the decision tree, at each node of the path
(and at no other node).
In the case when all features are boolean, the length of such 
a path is at most $n$, the number of features,
under the reasonable assumption that the same boolean 
feature is not redundantly tested two times on the same path.
Standard heuristic algorithms~\cite{BreimanFOS84,Utgoff89} 
which choose which feature to
branch on at each node according to a score, such as entropy,
will thus have a complexity of $O(n^2m)$,
where $m$ is the number of instances in $\D$.

\begin{proposition}\label{p:existence}
Let $\mbb{T}$ be a theory and $\D \subseteq \mbb{T}$. 
For any classifier $\kappa$, there exists a decision-tree 
classifier $\sigma$ on $\mbb{T}$ such that $\forall y \in \D$, $\sigma(y) = \kappa(y)$. 
\end{proposition}

We now introduce the novel coherent explainers, called \textit{surrogate explainers}, 
which generate weak abductive explanations from the entire feature space but using the 
predictions of a surrogate classifier. 

\begin{definition}
Let $\mbb{T}$ be a theory, $\D \subseteq \mbb{T}$, $\kappa$ and $\sigma$ two classifiers 
on $\mbb{T}$ such that $\forall y \in \D$, $\sigma(y) = \kappa(y)$.  
A \emph{surrogate explainer} of $\kappa$ is a function $\bLsu$ such that  
$\forall x \in \D$, 
$$\bLsu(\langle\mbb{T},\kappa,\D,x\rangle) = \bLdw(\langle\mbb{T},\sigma,\mathbb{F}(\mbb{T}),x).$$ 
$\bLDT$ denotes the surrogate explainer where $\sigma$ is a decision tree.
\end{definition}

\noindent \textbf{Remark:} Notice that the same surrogate classifier $\sigma$ is 
used to explain the predictions of $\kappa$ for all instances of a given dataset. However, since 
$\sigma$ and $\kappa$ may disagree on instances outside a dataset, $\bLsu$ may use different 
surrogate classifiers for different datasets.  


We show next that it is \textbf{tractable} to find a surrogate explainer $\bLDT$ which is grounded 
on a decision tree. The decision tree can be converted into a set of mutually-exclusive 
decision rules~\cite{Quinlan87} whose premises can be viewed as 
an explanation for all instances associated with the corresponding leaf.
The path from the root to a leaf is, however, often redundant
as an explanation~\cite{IzzaIM22}.
However, in polynomial time it is possible to construct 
an explainer that is complete on $\mbb{F}$, satisfies
coherence and such that explanations are weak abductive explanations of the 
decision tree. 

\begin{theorem}\label{dt:tractable}
A surrogate explainer $\bLDT$ can be found in polynomial time.
\end{theorem}

Recall that $\kappa$ is a black-box function. Thus
finding explanations valid over the whole feature space 
$\mathbb{F}(\mbb{T})$ requires exponential time. Furthermore, 
the set of dwAXp's of $\kappa$ over $\D$ do not satisfy Coherence. 
As we have just seen, one way to achieve tractability while
guaranteeing Coherence is to consider the explanations 
over $\mathbb{F}(\mbb{T})$ of a surrogate function $\sigma$.
The fact that $\sigma$ and $\kappa$ agree on $\D$ means that they have the
same dwAXp's. Of course, since we make what is essentially an arbitrary choice
of the function $\sigma$, rather than studying $\kappa$ over the whole feature space, 
it is clear that the explanations returned do not necessarily 
correspond to explanations of $\kappa$ over $\mathbb{F}(\mbb{T})$.
On the other hand, we show that for \emph{any} envelope $X$, 
there exists a surrogate function $\sigma_X$ 
(expressible as a decision list)
such that for all $x \in \D$, any $E \in X$ such that 
$E \subseteq x$ is a dwAXp of $\langle\mbb{T},\sigma_X,\mbb{F}(\mbb{T}),x\rangle$.

\begin{proposition} \label{prop:surrogate}
	Let $\mbb{T}$ be a theory, $\D \subseteq \mbb{F}(\mbb{T})$, 
	$\kappa$ a classifier and $X$ an envelope under $(\D,\kappa)$. There 
	exists a surrogate function $\sigma_X : \mbb{F}(\mbb{T}) \rightarrow \C$
	such that (1) $\sigma_X$ and $\kappa$ are equal on $\D$ and 
	(2) for all $x \in \D$, $E \in X$ and $E \subseteq x$ implies
	$E$ is a dwAXp of $\langle\mbb{T},\sigma_X,\mbb{F}(\mbb{T}),x\rangle$.
\end{proposition}

Observe that $\sigma_X$ (defined in the proof of Proposition~\ref{prop:surrogate}) 
can be described by a decision list containing at most $m$ rules (one for each instance in $D$).
Of course, in practice, there may be many redundant rules that can be eliminated from this list.
We could also express $\sigma_X$ as a decision tree
or as a d-DNNF, a language of boolean functions which also 
allows tractable explaining~\cite{HuangIICA022}.


Let us now give the list of axioms satisfied by surrogate explainers and hence show that they are 
coherent explainers. 


\begin{theorem}\label{thm:lsu}
The explainer $\bLsu$ satisfies Feasibility, Validity, Success and Coherence. 
It violates the remaining axioms.
\end{theorem}
\section{Worked Example}\label{secWE}

As an example, we study the well-known zoo dataset~\cite{zoo} which consists of 101 instances
corresponding to animals from a zoo. Each animal is described by 
16 features and belongs to one of
7 classes: Mammal, Bird, Reptile, Fish, Amphibian, Bug or Invertebrate.
For example, \emph{antelope} is a mammal and \emph{crow} is a bird.
We assume that the classifier $\kappa$ correctly classifies all 101 instances in the dataset.

In the case of a black-box classifier $\kappa$ over a large feature space, we cannot perform
an exhaustive search over all instances. 
First consider concise abductive explanations based on this dataset (cAXp's).
A cAXp of \emph{antelope} is $\{$(milk,1)$\}$, since all other animals in the dataset that give milk are also mammals, and a cAXp of \emph{crow} is $\{$(feathers,1)$\}$, since all other animals in the dataset with feathers are also birds. Although no animal in the dataset milks its young and has feathers, these explanations are consistent and hence $\bLda$ (same for $\bLdw$) does not satisfy coherence.



If we now look for coherent explanations, we find a size-14 irrefutable explanation for classifying \emph{antelope} as a mammal, namely
$\{$(hair,1), (feathers,0), (eggs,0), (milk,1), (airborne,0), (aquatic,0), (predator,0), (toothed,1), (backbone,1), (breathes,1), (venomous,0), (legs,4), (tail,1), (catsize,1)$\}$.
This is minimal in that all size-13 explanations obtained by deleting a literal from this explanation
are \emph{not} irrefutable. This demonstrates that there is a risk that irrefutable explanations may be very large
(as in this case involving 14 of the 16 features) and hence not very informative for the user. 

The solution embodied by the notion of irrefutable explanations
guarantees coherence but at the expense of the size
of the explanation. Another solution is the use of
a surrogate classifier as discussed in Section~\ref{sec:surr}.
We applied the classic ID3 algorithm~\cite{Quinlan86} to the zoo
dataset to obtain a decision tree with 100\% accuracy
on the dataset. Since the most discriminating feature is
milk which splits all mammals from all non-mammals
in the dataset, this is the feature associated with
the root of the decision tree found by ID3. The 
resulting unique AXp of \emph{antelope} for this decision tree
is simply $\{$(milk,1)$\}$. This size-1 explanation based
on a surrogate classifier is in stark contrast
with the size-14 irrefutable explanation. The explanation
for \emph{crow} is $\{$(milk,0), (feathers,1)$\}$, which is,
of course, coherent with the explanation $\{$(milk,1)$\}$ for \emph{antelope}.
\section{Discussion}\label{conclusion}

This paper presented a comprehensive study of explainers that generate sufficient reasons 
from samples, thereby advancing our understanding of their families, strengths and weaknesses. 
It proposed some basic axioms, or formal properties, and showed that some subsets uniquely define 
various families of explainers that all generate weak abductive explanations. In other words, 
it identified the unique (family of) explainers that satisfy a given subset of axioms.   

The reason of the diversity of families is that certain axioms are incompatible, 
making it tricky to define efficient sample-based functions that satisfy all the axioms. 
Indeed, it is not possible to define a function that returns concise (irreducible) 
explanations while guaranteeing at least one explanation for every instance and global 
coherence of explanations of all instances. Thus, one of the three properties should be abandoned. 
The right-to-an-explanation imposed by legislators \cite{regul21,WH22} impels to keep Success and 
the need of non-erroneous explanations suggests keeping Coherence. Thus, the compromise would be 
to sacrifice Irreducibility and accept explanations that contain unnecessary information. 
For the same reasons, it is reasonable to sacrifice Completeness in the incompatible set $\mathbf{I_2}$. 
These choices have been made by our novel family of Coherent explainers. 
On the positive side, some coherent explainers ($\bLr$ and $\bLsu$) allow tractable explaining. 
The worked example showed that irrefutable explanations by $\bLr$ may contain a large number of features 
compared to those provided by $\bLsu$. 
However, $\bLr$ \textbf{coincides} with $\bLw$ when it uses the whole feature space, 
while this is not the case of $\bLsu$ which, even if it is based on a dataset, generates weak abductive 
explanations from the whole feature space. Indeed, the equality $\bLsu(\bQ) = \bLw(\bQ)$ may not hold since  
the surrogate model $\sigma$ may differ from the original classifier $\kappa$ outside the dataset. 
Finally, as shown in Table~\ref{tab2}, both functions violate the Fidelity property which states 
that an explainer would recover \textbf{all} the "real" weak-abductive explanations, i.e., those generated 
from the feature space. Note that the two functions recover some but not necessarily all real explanations. 
This is not surprising since a sample-based approach reasons under incomplete information. 


\section{Related Work and Conclusion}\label{related-work}

There are several works in the AI literature on explaining classification models. Some of them explain the inherent reasoning of models (e.g., \cite{marquis22a,leite22,IgnatievS21,darwiche18}) while others consider classifiers as 
black-boxes and look for possible correlations between instances and the classes assigned to them. 
In this second category, explanations are generally feature-based, like \emph{sufficient reasons} (eg. 
\cite{Audemard23,darwiche20,DhurandharCLTTS18,IgnatievNM19,Ribeiro16,Ribeiro0G18}), 
\emph{contrastive/counterfactual} explanations (eg. \cite{DhurandharCLTTS18,abs-1905-12698,abs-1711-00399}), or \emph{semi-factuals} (eg. \cite{McCloy2002,IzzaIS024}). 
Our paper fits within the second category with a focus on sufficient reasons. However, unlike 
most works in the literature which explore the whole feature space to generate explanations, we 
follow a sample-based approach limiting search on a subset of feature space. The closest works 
to our are those from \cite{Amgoud23a,ecai23,Ribeiro16,Ribeiro0G18}. 
The authors in \cite{Ribeiro16,Ribeiro0G18} proposed the well-known LIME and Anchors models which 
generate explanations using a sample of instances that are in the neighborhood of the instance being explained. 
It has been shown in \cite{SETZ,Haned} that they may return incorrect explanations. 
The origin of this deficiency has been uncovered in \cite{Amgoud23a} where the author proved 
that it is not possible to define an explainer that provides subset-minimal 
abductive explanations and guarantees Success and Coherence.  
As a consequence of this negative result, various sample-based explainers have been proposed, 
some of which satisfy Coherence at the cost of Success  \cite{Amgoud23b} while others promote Success \cite{ecai23}.   

Our paper generalized the above negative result, introduced novel axioms and identified several other incompatible properties. 
It provided full characterizations of all families of explainers that can be defined, including one ($\bLc$) that ensures Success and Coherence. $\bLc$ is the first family of sample-based explainers that that satisfy both axioms. We proved that some of its instances are tractable. 

\vspace{0.1cm}
This papers lends itself to several developments. 
A challenging open problem is the definition of sample-based coherent explainers that recover as many "real" explanations as possible while being tractable. A possible solution would be the 
use of constraints on input data in order to solve conflicts between dwAXp's. 
Another perspective is an axiomatic study of sample-based contrastive explanations.

\begin{appendix}

\section{Proofs of Properties} 
\setcounter{property}{0}

\begin{property}\cite{ecai23}
	Let $\mbb{T} = \langle \F, \d, \C\rangle$ be a theory with $n = |\F|$, $\D \subseteq \mbb{F}(\mbb{T})$ 
	with $m = |\D|$ and $E \in \mbb{E}(\mbb{T})$. 
		Testing whether $E$ is a dwAXp can be achieved in $O(mn)$ time.
\end{property}
\begin{proof}
	Proved in \cite{ecai23}. 
\end{proof}

\begin{property}\cite{ecai23}
Let $\mbb{T} = \langle \F, \d, \C\rangle$ be a theory with $n = |\F|$, $\D \subseteq \mbb{F}(\mbb{T})$ 
with $m = |\D|$ and $E \in \mbb{E}(\mbb{T})$. 
		Finding a cAXp can be achieved in $O(mn^2)$ time.
\end{property}

\begin{proof}
	Proved in \cite{ecai23}. 
\end{proof}
\begin{property}
For any $\D \subseteq \mbb{F}(\mbb{T})$ and any classifier $\kappa$, 
$\D$ is coherent under ($\D,\kappa$). 	
\end{property}

\begin{proof}
Follows from the fact that instances are pairwise inconsistent by definition. 	
\end{proof}
\begin{property}
    The function $\bLt$ is a coherent explainer. 
\end{property}

\begin{proof}
Straightforward from the definition.
\end{proof}
\begin{property}
Let $\mbb{T}$ be a theory, $\D \subseteq \mbb{F}(\mbb{T})$, and $\kappa$ a classifier. 
For all $X , X' \subseteq \mbb{E}(\mbb{T})$, if $X$ and $X'$ are irrefutable envelopes under $(\D,\kappa)$, then $X = X'$.
\end{property}

\begin{proof}
	Let $\mbb{T}$ be a theory, $\D \subseteq \mbb{F}(\mbb{T})$, and $\kappa$ a classifier. 
	
	$\blacktriangleright$ Assume that $X$ and $X'$ are irrefutable envelopes under $(\D,\kappa)$. Suppose that $\exists E \in X\setminus X'$. From the definition of irrefutable envelope,  $\nexists E' \in \mbb{W}$ such that $\{E, E'\}$ is incoherent. 
	Hence, $X' \cup \{E\}$ satisfies the second condition of the same definition. This contradicts the maximality of $X'$. 
	So, $X \subseteq X'$. In the same way, we can show that $X' \subseteq X$
 and hence $X = X'$.
\end{proof}


\begin{property}
        $\bLr$ is a coherent explainer, and 
  for any question $\bQ = \langle \mbb{T}, \kappa, \D, x \rangle$, 
            $x \in \bLr(\bQ)$.
\end{property}

\begin{proof}
$\bLr$ is a coherent explainer since from Proposition~\ref{prop:irr-env}, 
$\mathtt{Irr}(\D,\kappa) \in \mathtt{Coh}(\D,\kappa)$ and 
for any question $\bQ$ where $\bQ = \langle \mbb{T}, \kappa, \D, x \rangle$, 
$\bLr(\bQ) = \{E \in \mathtt{Irr}(\D,\kappa) \mid E \subseteq x\}$.

The second property follows from the first item of Proposition~\ref{prop:irr-env}, 
namely the fact that $\D \subseteq \mathtt{Irr}(\D,\kappa)$.  
\end{proof}


\begin{property}
	Let $\bQ = \langle \mbb{T}, \kappa, \D, x \rangle$ be a question. 
	\begin{itemize}
		\item $x \in \bLdw(\bQ)$.
		\item $\bLda(\bQ) \subseteq \bLdw(\bQ)$.
	\end{itemize}	
\end{property}	

\begin{proof}
	The properties follow straightforwardly from the definitions. 
\end{proof}

\begin{property}
Let $\bL$ be an explainer that satisfies Completeness. For any 
question $\bQ = \langle \mbb{T}, \kappa, \D, x \rangle$, $x \in \bL(\bQ)$.
\end{property}

\begin{proof}
Let $\bL$ be an explainer that satisfies Completeness. 
Consider the question $\bQ = \langle \mbb{T}, \kappa, \D, x \rangle$, $x \in \bL(\bQ)$ and 
assume that $x \notin \bL(\bQ)$. Completeness of $\bL$ implies that $\exists y \in \D$ such that 
$x \subseteq y$ and $\kappa(x) \neq \kappa(y)$. From definition of instances, $x = y$ which is a 
contradiction.
\end{proof}
\section{Proofs of Propositions} 
\setcounter{prop}{0}
\begin{proposition}
The following implications hold. 
\vspace{-0.15cm}
\begin{itemize}
\item Completeness $\Rightarrow$ Strong Completeness, \\
Strong Completeness $\Rightarrow$ Success.
\item Irreducibility $\Rightarrow$ Strong Irreducibility.
\item Success, Feasibility and Coherence $\Rightarrow$ Validity.
\item Feasibility, Validity and Completeness $\Rightarrow$ Counter-Monotonicity. 
\end{itemize} 
\end{proposition}

\begin{proof}
	Let $\bL$ be an explainer. 

 
	$\blacktriangleright$ Assume that $\bL$ satisfies Completeness. 
    Let $E \subseteq x$ be such that $E \notin \bL(\bQ)$. From Completeness, 
    $\exists y \in \D$ s.t. $E \subseteq y$ and $\kappa(y) \neq \kappa(x)$. 
    Thus, $\exists y \in \mbb{F}(\mbb{T})$ s.t. $E \subseteq y$ and $\kappa(y) \neq \kappa(x)$. So, $\bL$ satisfies Strong Completeness. 
    
    Assume that $\bL$ satisfies Strong Completeness and that $x \notin \bL(\bQ)$ such that $\bQ = \langle \mbb{T}, \kappa, \D, x \rangle$. From Strong Completeness, 
    $\exists y \in \mbb{F}(\mbb{T})$ such that $x \subseteq y$ and $\kappa(x) \neq \kappa(y)$. By definition of instances, $x = y$. Since $\kappa$ assigns a single 
    value to every instance, it follows that $\kappa(x) = \kappa(y)$. 
    Thus $x \in \bL(\bQ)$  and $\bL(\bQ) \neq \emptyset$. 
    
    $\blacktriangleright$ Strong Irreducibility follows straightforwardly from 
    Irreducibility since $\D \subseteq \mbb{F}(\mbb{T})$. 
 
	$\blacktriangleright$ Assume that $\bL$ satisfies Success, Feasibility, Coherence and violates Validity. 
	Then, there exists a question $\bQ = \langle \mbb{T}, \kappa, \D, x \rangle$ such that  
	$\exists E \in \bL(\bQ)$ and $\exists y \in \D$ such that $E \subseteq y$ and $\kappa(x) \neq \kappa(y)$.  
	Consider now the question $\bQ' = \langle \mbb{T}, \kappa, \D, y \rangle$. 
	From Success, $\exists E' \in \bL(\bQ')$. From Feasibility, $E' \subseteq y$ so $E \cup E' \subseteq y$. 
	Since $y$ is consistent, $E \cup E'$ is consistent \textbf{(a)}. 
	From Coherence of $\bL$, since  $\kappa(x) \neq \kappa(y)$, it follows that $E \cup E'$ is inconsistent, which contradicts \textbf{(a)}.

	$\blacktriangleright$ Assume that $\bL$ satisfies Feasibility, Validity and Completeness and violates Counter-Monotonicity. 
	Hence, there exist two samples $\D, \D'$ such that $\D \subseteq \D'$, $\exists x \in \D$ such that 
	$\exists E \in \bL(\bQ')$ while  $E \notin \bL(\bQ)$, with $\bQ = \langle \mbb{T}, \kappa, \D, x \rangle$ and $\bQ' = \langle \mbb{T}, \kappa, \D', x \rangle$. From Feasibility, $E \subseteq x$. 
	Since $E \notin \bL(\bQ)$, then from Completeness of $\bL$, $\exists y \in \D$ such that $E \subseteq y$ and 
	$\kappa(x) \neq \kappa(y)$ (\textbf{a}). 
	Validity of $\bL$ ensures that $\nexists z \in \D'$ such that $E \subseteq z$ and 	$\kappa(x) \neq \kappa(z)$. 
	Since $\D \subseteq \D'$, we deduce that $y \in \D'$ which contradicts (\textbf{a}). 
\end{proof}	


\begin{proposition}
	Let $\mbb{T}$ be a theory, $\D \subseteq \mbb{F}(\mbb{T})$, and $\kappa$ a classifier. 
	\begin{itemize} 
		\item Let $X \in \mathtt{Coh}(\D,\kappa)$. For any $x \in \D$, for any $E \in X$, 
		if $E \subseteq x$, then $E \in \bLdw(\langle\mbb{T},\kappa,\D, x\rangle)$.
		\item $\D \in \mathtt{Coh}(\D,\kappa)$. 
            \item For any $E \in \bLdw(\langle\mbb{T},\kappa,\D, x\rangle)$, 
        $\exists X \in \mathtt{Coh}(\D,\kappa)$ such that $X$ is a subset-maximal envelope of $\mathtt{Coh}(\D,\kappa)$ and $E \in X$. 
	\end{itemize}
\end{proposition}

\begin{proof}
Let $\mbb{T}$ be a theory, $\D \subseteq \mbb{F}(\mbb{T})$ and $\kappa$ a classifier. 

$\blacktriangleright$ Let $X \in \mathtt{Coh}(\D,\kappa)$, $x \in \D$, and $E \in X$.  
Assume that $E \subseteq x$. 
From the second condition of the definition of an envelope, $\exists y \in \D$ such that 
$E \in \bLdw(\langle\mbb{T},\kappa,\D, y\rangle)$. From coherence of $X$, 
$\kappa(x) = \kappa(y)$, so $E \in \bLdw(\langle\mbb{T},\kappa,\D, x\rangle)$.

$\blacktriangleright$ From Property~\ref{p-da-dw}, $\forall x \in \D$, $x \in \bLdw(\langle\mbb{T},\kappa,\D, x\rangle)$, so $\D$ satisfies the second condition of the definition of an envelope.
$\D$ also trivially satisfies the third condition of this same definition.
From Property~\ref{sample-coh}, $\D$ is coherent, and hence satisfies 
the first condition of the same definition. So $\D \in \mathtt{Coh}(\D,\kappa)$. 

$\blacktriangleright$ Let $E \in \bLdw(\langle\mbb{T},\kappa,\D, x\rangle)$. We show that 
we can build a subset-maximal envelope $X$ which contains $E$. 
From Property~\ref{sample-coh}, $\D$ is coherent. Then, we start with 
$X_0 = \D$. Assume that $X_1 = X_0 \cup \{E\}$ is incoherent. Then, 
$\exists y \in \D$ such that $E \cup y$ is consistent and $\exists z \in \D$ with 
$E \subseteq z$ s.t. $\kappa(y) \neq \kappa(z)$. 
Since $y$ is an instance, then $E \subseteq y$. From  
$E \in \bLdw(\langle\mbb{T},\kappa,\D, x\rangle)$, it follows that  
$\kappa(y) = \kappa(z) = \kappa(x)$, which contradicts the incoherence of $X_1$. 
We then add to $X_1$ all dwAXp's built from $\D$ up to coherence.
\end{proof}


\begin{proposition}
Let $\mbb{T}$ be a theory, $\D \subseteq \mbb{F}(\mbb{T})$, and $\kappa$ a classifier. 
\begin{itemize}
	\item $\D \subseteq \mathtt{Irr}(\D,\kappa)$.   
    \item $\mathtt{Irr}(\D,\kappa) \in \mathtt{Coh}(\D,\kappa)$.
    \item $\mathtt{Irr}(\D,\kappa) = \bigcap\limits_{X_i \in \mathcal{S}} X_i$, 
    where $\mathcal{S}$ is the set of all subset-maximal envelopes of $\mathtt{Coh}(\D,\kappa)$.
\end{itemize}
\end{proposition}

\begin{proof}
	Let $\mbb{T}$ be a theory, $\D \subseteq \mbb{F}(\mbb{T})$, and $\kappa$ a classifier. 
	
	$\blacktriangleright$ Assume that $x \in \D$ and $x \notin \mathtt{Irr}(\D,\kappa)$. 
	From Property~\ref{p-da-dw}, $x \in \bLdw(\langle \mbb{T}, \kappa, \D, x \rangle)$.
    Then, there exists a question $\bQ$ and $\exists E \in \bLdw(\bQ)$ such that $\{E, x\}$ is incoherent. 
    So, $E \cup x$ is consistent and $\exists z \in \D$ such that $E \subseteq z$ and $\kappa(x) \neq \kappa(z)$. 
    From consistency of   $E \cup x$, it follows that $E \subseteq x$; which contradicts the fact that $E$ is a dwAXp.

    $\blacktriangleright$ From Def.~10, $\mathtt{Irr}(\D,\kappa) \subseteq \mbb{E}(\mbb{T})$ such that it is coherent and $\mathtt{Irr}(\D,\kappa) \subseteq \bigcup\limits_{x \in \D}\bLdw(\langle \mbb{T}, \kappa, \D, x \rangle)$. From the above result 
    (i.e., $\D \subseteq \mathtt{Irr}(\D,\kappa)$), it follows that 
    $\mathtt{Irr}(\D,\kappa) \in \mathtt{Coh}(\D,\kappa)$. 
    
    $\blacktriangleright$ Let us show the third property. 
    Let $\mathbb{W} = \bigcup\limits_{x \in \D}\bLdw(\langle\mbb{T}, \kappa, \D, x\rangle)$ and 
    $\mathcal{S}$ be the set of all subset-maximal envelopes of $\mathtt{Coh}(\D,\kappa)$. 
    From Proposition~\ref{prop:coh}, $\D$ is an envelope. So  $\mathtt{Coh}(\D,\kappa) \neq \emptyset$. 
    Let $\mathcal{S} = \{X_1, \ldots, X_k\}$ be the set of all subset-maximal envelopes. 
    From Proposition~\ref{prop:coh}, $\forall i = 1,\ldots,k$, $X_i \subseteq \mathbb{W}$. 
    
    Assume that $E \in \mathtt{Irr}(\D,\kappa)$. Hence, $E \in \mathbb{W}$ and 
    $\forall E' \in \mathbb{W}$, $\{E,E'\}$ is coherent. 
    From Coherence of every envelope $X_i$, $i = 1,\ldots,k$, it follows that $X_i \cup \{E\}$ is 
    coherent and from Maximality of $X_i$, $E \in X_i$,  
    thus $E \in \bigcap\limits_{X_i \in \mathcal{S}} X_i$. 

    Assume now that $E \in \bigcap\limits_{X_i \in \mathcal{S}} X_i$ and 
    $E \notin \mathtt{Irr}(\D,\kappa)$. From the subset-maximality of $\mathtt{Irr}(\D,\kappa)$, 
    $\exists E' \in \mathbb{W}$ such that $\{E,E'\}$ is incoherent. 
    From Proposition~\ref{prop:coh}, $\exists X_j \in \mathcal{S}$ such that $E' \in X_j$. From the assumption 
        $E \in \bigcap\limits_{X_i \in \mathcal{S}} X_i$, it holds that $E \in X_j$, which 
        contradicts the coherence of the envelope $X_j$. 
\end{proof}	


\begin{proposition}
Let $\mbb{T}$ be a theory and $\D \subseteq \mbb{T}$. 
For any classifier $\kappa$, there exists a decision-tree 
classifier $\sigma$ on $\mbb{T}$ such that $\forall y \in \D$, $\sigma(y) = \kappa(y)$. 
\end{proposition}

\begin{proof}
Let $\D$ be a set of $m$ distinct feature-vectors. A decision tree composed of at most $m-1$ 
nodes can always be built to distinguish between the $m$ elements of $\D$, with
each $x \in \D$ corresponding to a leaf $\ell_x$.
It then suffices to associate to each leaf $\ell(x)$ the corresponding value $\kappa(x)$.
\end{proof}

\begin{proposition} 
	Let $\mbb{T}$ be a theory, $\D \subseteq \mbb{F}(\mbb{T})$, 
	$\kappa$ a classifier and $X$ an envelope under $(\D,\kappa)$. Then
	there exists a surrogate function $\sigma_X : \mathbb{F} \rightarrow \C$
	such that (1) $\sigma_X$ and $\kappa$ are equal on $\D$ and 
	(2) for all $x \in D$, $E \in X$ and $E \subseteq x$ implies
	$E$ is a weak abductive explanation of $\langle\mbb{T},\sigma_X,\mathbb{F},x\rangle$.
\end{proposition}

\begin{proof}
	It suffices to define $\sigma_X$ and show that it has the desired properties.
	Let $\mathbb{F}_X \subseteq \mathbb{F}$ be the instances $z \in \mathbb{F}$
	such that $E \subseteq z$ for some $E \in X$. 
	For $z \in \mathbb{F}_X$, define $\sigma_X(z)$ 
	to be equal to $\kappa(x)$ if there exists $E \in X$ and $x \in \D$ such that
	$E \subseteq \mathbf{z}$ and $E \subseteq x$. 
	By the definition of an envelope, $E$ is a weak abductive explanation of
	$\kappa$ on $\D$ and so the value of $\kappa(x)$ is identical
	for all $x \in \D$ such that $E \subseteq x$. 
	Furthermore, if $E \subseteq z$ and $E' \subseteq z$, 
	for $E,E' \in X$,
	and $E \subseteq x$, $E' \subseteq x'$, where $x,x' \in \D$,
	then we must have $\kappa(x)=\kappa(x')$ by coherence of $X$.
	Hence $\sigma_X(z)$ is well-defined.
	For $z \in \mathbb{F} \setminus \mathbb{F}_E$,
	let $\sigma_X(z)$ be equal to an arbitrary value from $\C$.
	
	For $x \in \D$, $\sigma_X(x)=\kappa(x)$ 
	since by definition of an envelope,
	there exists $E \in X$ such that $E \subseteq x$.
	Furthermore, for all $x \in \D$ and for all $E \in X$
	such that $E \subseteq x$, $E$ is a weak abductive explanation of $\langle\mbb{T},\sigma,\mathbb{F},x\rangle$.
	since for all $z$ such that $E \subseteq z$, we have $\sigma_X(z)=\kappa(x)$ (by definition of $\sigma_X$).
\end{proof}



\section{Proofs of Theorems} 

\setcounter{theorem}{0}

\begin{theorem}
The axioms of every set $\mathbf{I}_{i=1,5}$ below are incompatible.
\begin{description}
\item [$(\mathbf{I}_1)$] Feasibility, Success, Coherence and Irreducibility. 
\item [$(\mathbf{I}_2)$] Feasibility, Coherence and Completeness. 
\item [$(\mathbf{I}_3)$] Strong Irreducibility and Strong Completeness.
\item [$(\mathbf{I}_4)$] Feasibility, Validity, Success, Irreducibility and Monotonicity.  
\item [$(\mathbf{I}_5)$] Feasibility, Validity, Success, Irreducibility and CM. 
\end{description}	
\end{theorem}

\begin{proof}
$\ $

$\blacktriangleright$ To show the incompatibility of the axioms in $\mathbf{I}_1$ (Feasibility, Success, Coherence and Irreducibility), let us consider the theory 
and dataset $\D_1$ below. 
\begin{center}
	\begin{tabular}{c|cc|c}\hline
		$\D_1$ & $f_1$ & $f_2$ & $\kappa_1(x_i)$   \\\hline
		$x_1$  &   0   &   1   &   0             \\
		$x_2$  &   1   &   0   &   1             \\
		$x_3$  &   0   &   0   &   2             \\\hline
	\end{tabular}
\end{center}
Let us focus on the two instances $x_1$ and $x_2$. 
The dwAXps of $\bQ_1 = \langle \mbb{T}, \kappa_1, \D_1, x_1 \rangle$ and 
$\bQ_2 = \langle \mbb{T}, \kappa_1, \D_1, x_2 \rangle$ are: 
\begin{itemize}
	\item [] $\bLdw(\bQ_1) = \{E_1, E_2\}$  \hfill 
	         $\bLdw(\bQ_2) = \{E'_1, E'_2\}$ 
\end{itemize}
\begin{itemize}
	\item [] $E_1 = \{(f_2,1)\}$               \hfill $E'_1 = \{(f_1, 1)\}$
	\item [] $E_2 = \{(f_1, 0), (f_2,1)\}$     \hfill $E'_2 = \{(f_1, 1), (f_2, 0)\}$.
\end{itemize}

\noindent Let now $\bL$ be an explainer that satisfies Success, Feasibility, Coherence and Irreducibility.  

\noindent From Success, $\bL(\bQ_1) \neq \emptyset$ and $\bL(\bQ_2) \neq \emptyset$. 

\noindent From Feasibility, Validity and Theorem~\ref{th-cara}, 
$\bL(\bQ_1) \subseteq \bLdw(\bQ_1)$ and $\bL(\bQ_2) \subseteq \bLdw(\bQ_2)$.

\noindent From Irreducibility, $\bL(\bQ_1) = \{E_1\}$ and $\bL(\bQ_2) = \{E'_1\}$.  
Note that $E_1 \cup E'_1$ is consistent and $\kappa(x_1) \neq \kappa(x_2)$, which 
contradicts the fact that $\bL$ satisfies Coherence. 


$\blacktriangleright$ To show the incompatibility of the axioms in $\mathbf{I}_2$ (Feasibility, Coherence and Completeness), 
assume an explainer $\bL$ that satisfies the three axioms and 
consider the above theory, classifier and dataset $\D_1$. Proposition~\ref{links} ensures that $\bL$ satisfies 
Success (from Completeness) and Validity (from Success Feasibility and Coherence).  
Theorem~\ref{th:dw} implies that $\bL = \bLdw$, so  
$\bL(\bQ_1) = \{E_1, E_2\}$  and
$\bL(\bQ_2) = \{E'_1, E'_2\}$.  
Note that $E_1 \cup E'_1$ is consistent while $\kappa(x_1) \neq \kappa(x_2)$, which contradicts Coherence of $\bL$.

$\blacktriangleright$ To show the incompatibility of the axioms in $\mathbf{I}_3$ (Strong Irreducibility and Strong Completeness), assume an explainer $\bL$ that 
satisfies them both and  
consider the classifier $\kappa_1$, dataset
$\D_1$ and question $\bQ_1 = \langle \mbb{T}, \kappa_1, \D_1, x_1 \rangle$ (above). 
We complete the definition of $\kappa_1$ over the whole feature space by setting $\kappa_1(1,1)=0$.
Strong Completeness tells us that 
$E_1,E_2 \in \bL(\bQ_1)$, but this violates Strong 
Irreducibility since $E_1 \subset E_2$.


$\blacktriangleright$ To show the incompatibility of the axioms in $\mathbf{I}_4$ 
(Feasibility, Validity, Success, Irreducibility, and Monotonicity), assume an explainer 
$\bL$ that satisfies them all, and consider the theory, classifier and dataset below.

\begin{multicols}{2}
\begin{center}
\begin{tabular}{c|cc|c}\hline
		$\D_1$ & $f_1$ & $f_2$ & $\kappa_3(x_i)$   \\\hline
	\rowcolor{maroon!10}	$x_1$  &   0   &   0   &   0             \\
		$x_2$  &   0   &   1   &   1             \\
  \hline
\end{tabular}
\end{center}
	\begin{itemize}
    \item [] \qquad $E_1 = \emptyset$ 
	\item [] \qquad $E_2 = \{(f_1,0)\}$ 
	\item [] \qquad $E_3 = \{(f_2,0)\}$
	\item [] \qquad $E_4 = x_1$
\end{itemize}
\end{multicols}

\noindent Consider now the question $\bQ = \langle\mbb{T},\kappa_3,\D_1,x_1\rangle$.  
From Feasibility, $\bL(\bQ) \subseteq \{E_1, E_2, E_3, E_4\}$. 
From Validity, $E_1 \notin \bL(\bQ)$ and $E_2 \notin \bL(\bQ)$. 
From Irreducibility, $E_4 \notin \bL(\bQ)$. 
From Success, $\bL(\bQ) = \{E_3\}$.  
Consider now the extended dataset $\D_2$ below and the question 
$\bQ' = \langle\mbb{T},\kappa_3,\D_2,x_1\rangle$. 
\begin{center}
\begin{tabular}{c|cc|c}\hline
		$\D_2$ & $f_1$ & $f_2$ & $\kappa_3(x_i)$   \\\hline
  \rowcolor{maroon!10}	$x_1$  &   0   &   0   &   0             \\
		$x_2$  &   0   &   1   &   1             \\
        $x_3$  &   1   &   0   &   1             \\
  \hline
\end{tabular}
\end{center}
 
\noindent From Monotonicity, since $\D_1 \subset \D_2$, then $E_3 \in \bL(\bQ')$ while 
Validity ensures that $E_3 \notin \bL(\bQ')$.

$\blacktriangleright$ To show the incompatibility of the axioms in $\mathbf{I}_5$ 
(Feasibility, Validity, Success, Irreducibility, and CM), it is sufficient to consider 
the above example. Note that $\bL(\bQ) = \{E_3\}$ and $\bL(\bQ') = \{E_4\}$ while 
$\D_1 \subseteq \D_2$. 
\end{proof}

\begin{theorem}
The axioms of every set $\mathbf{C}_{i=1,5}$ below are compatible.
\begin{description}
\item [$(\mathbf{C}_1)$] Feasibility, Validity, Success, (Strong) Completeness, CM. 
\item [$(\mathbf{C}_2)$] Feasibility, Validity, Success, (Strong) Irreducibility. 
\item [$(\mathbf{C}_3)$] Feasibility, Validity, Success, Coherence, Monotonicity, CM, Strong Completeness.  
\item [$(\mathbf{C}_4)$] Feasibility, Validity, Success, Coherence, Monotonicity, CM,  
                         Strong Irreducibility.                         
\item [$(\mathbf{C}_5)$] Feasibility, Validity, Coherence, (Strong) Irreducibility,
Monotonicity, CM. 
\end{description}	
\end{theorem}

\begin{proof}
$\ $

$\blacktriangleright$ The axioms of the set $\mathbf{C}_1$ (Feasibility, Validity, Success, (Strong) Completeness and Counter-Monotonicity) are compatible as $\bLdw$ satisfies all of them (see Table~\ref{tab2}).

$\blacktriangleright$ The axioms of the set $\mathbf{C}_2$ (Feasibility, Validity, Success, (Strong) Irreducibility) are compatible since they are all satisfied by the explainer $\bLda$ (see Table~\ref{tab2}). 

$\blacktriangleright$ The  axioms of the set $\mathbf{C}_3$ (Feasibility, Validity, Success, Coherence, Monotonicity, CM, Strong Completeness) are compatible since the explainer $\bLw$ satisfies all of them (see Table~\ref{tab2}).

$\blacktriangleright$ The  axioms of the set $\mathbf{C}_4$ (Feasibility, Validity, Success, Coherence, Monotonicity, CM, Strong irreducibility) are compatible since the explainer $\bLa$ satisfies all of them (see Table~\ref{tab2}).

$\blacktriangleright$ The  axioms of the set $\mathbf{C}_5$ (Feasibility, Validity, Coherence, (Strong) Irreducibility, Monotonicity, CM) are compatible since the trivial explainer which returns the empty set for each question satisfies all of them. 
\end{proof}
\begin{theorem}
Considering all sets of axioms among those listed in Table~1
which include Feasibility and Validity, the only minimal incompatible sets  
are $\mathbf{I}_1,\ldots,\mathbf{I}_5$ and the only maximal
compatible sets are $\mathbf{C}_1,\ldots,\mathbf{C}_5$.
\end{theorem}

\begin{proof}
Let $\mathbf{C}$ be a compatible set of axioms satisfied by an explainer $\bL$. The following five cases are mutually exclusive and cover all possible sets $\mathbf{C}$:
\begin{enumerate}
\item If $\mathbf{C}$ contains Completeness, then by Theorem~\ref{th:dw}, $\bL =\bLdw$ and by Theorem~\ref{th:bldw}, $\bL$ satisfies the axioms $\mathbf{C_1}$. It follows trivially that $\mathbf{C_1}$ is a maximal compatible set of axioms.
\item If $\mathbf{C}$ contains neither Completeness nor Success, then by Proposition~\ref{links} it cannot contain Strong Completeness. All axioms apart from these three (i.e. $\mathbf{C_5}$) are satisfied by the explainer $\bL$ that returns the empty set for each question. It follows that $\mathbf{C_5}$ is a maximal compatible set of axioms.
\item If $\mathbf{C}$ contains Success and Irreducibility but not Completeness, then by Proposition~\ref{links}, $\bL$ also satisfies Strong Irreducibility and hence the set of axioms $\mathbf{C_2}$. By the incompatibilities in Theorem~\ref{th:incompatibility}, $\bL$ can satisfy none of Coherence, Strong Completeness, Monotonicity and CM. Hence $\mathbf{C_2}$ is a maximal compatible set of axioms.
\item If $\mathbf{C}$ contains Success and Strong Irreducibility but neither Completeness nor Irreducibility, then by Theorem~\ref{th:incompatibility} it cannot contain Strong Completeness. The explainer $\bLa$ satisfies all other axioms (i.e. $\mathbf{C_4}$). It follows that  $\mathbf{C_4}$ is a maximal compatible set of axioms.
\item The final case is when $\mathbf{C}$ contains Success but neither Completeness nor (Strong) Irreducibility. All other axioms (i.e. $\mathbf{C_3}$) are satisfied by the explainer $\bLw$ It follows that $\mathbf{C_3}$ is a maximal compatible set of axioms.
\end{enumerate}
The minimality of the incompatible sets $\mathbf{I_i}$ ($i=1,\ldots,5$) is easily verified by checking that each proper subset of $\mathbf{I_i}$ is a subset of one of the $\mathbf{C_j}$ ($j=1,\ldots,5$) and hence is compatible.
The above case analysis only required the incompatibility of $\mathbf{I_1},\ldots,\mathbf{I_5}$ to cover all possible sets of axioms; it follows that there are no other minimal incompatible sets of axioms.
\end{proof}

\begin{theorem}
An explainer $\bL$ satisfies Feasibility and Validity \textbf{iff} for any question 
$\bQ = \langle \mbb{T}, \kappa, \D, x \rangle$, $\bL(\bQ) \subseteq \bLdw(\bQ)$.	
\end{theorem}

\begin{proof}
	Let $\bL$ be an explainer. 
	
	$\blacktriangleright$Assume that $\bL$ that satisfies Feasibility and Validity. 
	Let $\bQ = \langle \mbb{T}, \kappa, \D, x \rangle$ be a question and $E \in \bL(\bQ)$. 
	Feasibility of $\bL$ implies $E \subseteq x$. Validity ensures that 
	$\forall y \in \D$ such that $E \subseteq y$, it holds that $\kappa(x) = \kappa(y)$. So, $E$ is a dwAXp, 
	i.e., $\bL(\bQ) \subseteq \bLdw(\bQ)$.
	
	$\blacktriangleright$ Assume that for any question 
	$\bQ = \langle \mbb{T}, \kappa, \D, x \rangle$, $\bL(\bQ) \subseteq \bLdw(\bQ)$. 
	Then, $\forall E \in \bL(\bQ)$, the following hold: 
	i) $E \subseteq x$, so $\bL$ satisfies Feasibility, and 
	ii) $\forall y \in \D$, if $E \subseteq y$, then $\kappa(x) \neq \kappa(y)$, so $\bL$ satisfies Validity.  
\end{proof}	
\begin{theorem}
An explainer $\bL$ satisfies Feasibility, Validity and Completeness \textbf{iff} 
$\bL = \bLdw$.
\end{theorem}

\begin{proof}
Let $\bL$ be an explainer. 

$\blacktriangleright$  Assume that $\bL = \bLdw$. 
From Table~\ref{tab2}, 	$\bLdw$ satisfies Feasibility, Validity and Completeness.
 
$\blacktriangleright$  Assume now $\bL$ is an explainer that satisfies Feasibility, Validity and Completeness. 
Let $\bQ = \langle \mbb{T}, \kappa, \D, x \rangle$ be a question. 
\begin{itemize}
	\item From Theorem~\ref{th-cara}, $\bL(\bQ) \subseteq  \bLdw(\bQ)$.
	\item Let $E \in \bLdw(\bQ)$ and assume $E \notin \bL(\bQ)$. Completeness of $\bL$ implies that $\exists y \in \D$ 
	      such that 
	      $E \subseteq y$ and $\kappa(x) \neq \kappa(y)$, which contradicts the validity of $\bLdw$.
	      So, $\bLdw(\bQ) \subseteq  \bL(\bQ)$. 
\end{itemize}
Thus, $\bL(\bQ) = \bLdw(\bQ)$.  
\end{proof}

\begin{theorem}
The explainer $\bLdw$ satisfies Feasibility, Validity, Success, (Strong) Completeness and CM. 
It violates the remaining axioms.
\end{theorem}

\begin{proof} Let us show the properties satisfied by the function ${\bLdw}$.
	
	$\blacktriangleright$ \textit{Feasibility and Validity} of ${\bLdw}$ follow 
    from its definition.   

        $\blacktriangleright$ \textit{Success} of ${\bLdw}$ follows from Property~\ref{p-da-dw} (every instance $x$ is a weak abductive explanation of $\kappa(x)$).
	

	$\blacktriangleright$ 
	\textit{Counter-Monotonicity:} 
	Let $\bQ = \langle \mbb{T}, \kappa, \D, x \rangle$ and 
	$\bQ' = \langle \mbb{T}, \kappa, \D', x \rangle$ be two questions such that   
	$\D \subseteq \D' \subseteq \mbb{F}(\mbb{T})$. 
	Let $E \in {\bLdw}(\bQ')$. Hence, $E \subseteq x$ and $\forall y \in \D'$, 
	if $E \subseteq y$, then $\kappa(y) = \kappa(x)$. 
	Since $\D \subseteq \D'$, we have $\{y \in \D \mid E \subseteq y\} \subseteq \{y \in \D' \mid E \subseteq y\}$. 
	Hence $E \in {\bLdw}(\bQ)$.
	
	$\blacktriangleright$ 	 
	\textit{(Strong) Completeness:} Let us show that ${\bLdw}$ satisfies Completeness. 	
	Let $\bQ = \langle \mbb{T}, \kappa, \D, x \rangle$ and $E \subseteq x$. Assume that 
	$E \notin \bLdw(\bQ)$. Then, $E$ violates the validity condition in the definition of dwAXp's, 
	hence $\exists y \in \D$ such that $E \subseteq y$ and $\kappa(y) \neq \kappa(x)$. 
	So, $\bL$ satisfies Completeness.	
        Strong Completeness follows from Completeness by Proposition~\ref{links}.

   	$\blacktriangleright$ \textit{Coherence:} 
    From Theorem~\ref{th:incompatibility}, since ${\bLdw}$ satisfies Feasibility and Completeness, 
    we can deduce that Coherence is violated (the set $\mathbf{I}_2$ being incompatible).

	$\blacktriangleright$ \textit{Irreducibility:} From Theorem~\ref{th:incompatibility}, 
    since ${\bLdw}$ satisfies Feasibility and Completeness, 
    it violates Irreducibility (the set $\mathbf{I}_3$ being incompatible). 
 
	$\blacktriangleright$ 
	\textit{Monotonicity:}  Consider the above example using the classifier $\kappa_3$. 
    Note that $\bLdw(\bQ) = \{E_3\}$ and $\bLdw(\bQ') = \{E_4\}$ while 
    $\D_1 \subseteq \D_2$. 
\end{proof}
\begin{theorem}
An explainer $\bL$ satisfies Feasibility, Validity, Monotonicity and Strong Completeness  \textbf{iff} $\bL = \bLw$.
\end{theorem}

\begin{proof}
Let $\bL$ be an explainer. 

$\blacktriangleright$ Assume that $\bL$ satisfies Feasibility, Validity, Monotonicity and Strong Completeness. Let $\bQ = \langle \mbb{T}, \kappa, \D, x \rangle$ and 
$\bQ' = \langle \mbb{T}, \kappa, \mbb{F}(\mbb{T}), x \rangle$. 
Since $\D \subseteq \mbb{F}(\mbb{T})$, then from Monotonicity of $\bL$, 
$\bL(\bQ) \subseteq \bL(\bQ')$. 
From  Feasibility and Validity, it follows from Theorem~\ref{th-cara} that 
$\bL(\bQ') \subseteq \bLdw(\bQ')$. By definition of $\bLw$, 
$\bLw(\bQ) = \bLdw(\bQ')$. So, $\bL(\bQ) \subseteq \bLw(\bQ)$. 

Consider now some $E \in \bLw(\bQ)$ and $E \notin \bL(\bQ)$. 
Then, $E \subseteq x$. From Strong Completeness, $\exists y \in \mbb{F}(\mbb{T})$ 
such that $E \subseteq y$ and $\kappa(x) \neq \kappa(y)$. This contradicts the 
validity of $E$ on $\mbb{F}(\mbb{T})$. So, $\bLw(\bQ) \subseteq \bL(\bQ)$.

$\blacktriangleright$ Theorem~\ref{th:blw} shows that $\bLw$ satisfies the four axioms.
\end{proof}
\begin{theorem}
The explainer $\bLw$ satisfies Feasibility, Validity, Success, Strong Completeness, Coherence, 
Monotonicity and CM. It violates (Strong) Irreducibility and Completeness.
\end{theorem}

\begin{proof}
$\ $

$\blacktriangleright$ \emph{Feasibility} and \emph{Validity} follow from the definition. 

$\blacktriangleright$ \emph{Success} follows from the fact that for any $x \in \mbb{F}(\mbb{T})$, 
$x \in \bLw(\bQ)$ where $\bQ = \langle \mbb{T}, \kappa, \mbb{F}(\mbb{T}), x \rangle$. 

$\blacktriangleright$ \emph{Coherence}: Let 
$\bQ = \langle \mbb{T}, \kappa, \D, x \rangle$ and 
$\bQ' = \langle \mbb{T}, \kappa, \D, y \rangle$ such that $\kappa(x) \neq \kappa(y)$.   
Assume that $E \in \bLw(\bQ)$, $E' \in \bLw(\bQ')$ and $E \cup E'$ is consistent. 
Then, $\exists z \in \mbb{F}(\mbb{T})$ such that $E \cup E' \subseteq z$. 
From Validity of $\bLw$, $\kappa(x) = \kappa(z)$ and $\kappa(y) = \kappa(z)$, 
hence $\kappa(x) = \kappa(y)$, which contradicts the assumption.

$\blacktriangleright$ \emph{Monotonicity and Counter-Monotonicity} follow from 
the fact that for all $\D \subseteq \D'$, for all 
$\bQ = \langle \mbb{T}, \kappa, \D, x \rangle$ and 
$\bQ' = \langle \mbb{T}, \kappa, \D', x \rangle$, 
$\bLw(\bQ) = \bLw(\bQ')$.

$\blacktriangleright$ \emph{Strong Completeness:} Let $\bQ = \langle \mbb{T}, \kappa, \D, x \rangle$ 
be a question and $E \subseteq x$. 
Assume $E \notin \bLw(\bQ)$. By definition of $\bLw$, 
$\exists y \in \mbb{F}(\mbb{T})$ such that $E \subseteq y$ and 
$\kappa(x) \neq \kappa(y)$.   Hence, $\bLw$ satisfies Strong Completeness.

$\blacktriangleright$ \emph{Irreducibility} and \emph{Completeness} are violated since 
from Theorem~\ref{th:incompatibility}, Completeness is incompatible with the pair 
(Feasibility, Coherence) and Irreducibility is incompatible with the tuple (Feasibility, 
Success, Coherence). Strong Irreducibility is violated since from Proposition~\ref{links}, it is incompatible with Strong Completeness.
\end{proof}
\begin{theorem}
An explainer $\bL$ satisfies Feasibility, Validity, and Irreducibility iff for any question 
$\bQ = \langle \mbb{T}, \kappa, \D, x \rangle$, $\bL(\bQ) \subseteq \bLda(\bQ)$.
\end{theorem}

\begin{proof}
	Let $\bL$ be an explainer. 
		
	$\blacktriangleright$  Assume now $\bL$ is an explainer that satisfies Feasibility, Validity, and Irreducibility. Let $\bQ = \langle \mbb{T}, \kappa, \D, x \rangle$ be a question. 
    Assume $E \in \bL(\bQ)$. From Theorem~\ref{th-cara}, $E \in \bLdw(\bQ)$, so $E$ is a weak abductive explanation. Assume now that $E \notin \bLda(\bQ)$, so $E$ is not subset-minimal. 
    Then, $\exists E' \in \bLdw(\bQ)$ such that $E' \subset E$. 
    hence, $\exists l \in E\setminus E'$ and $\exists y \in \D$ such that $E\setminus\{l\} \subseteq y$ and 
    $\kappa(x) = \kappa(y)$, which contradicts Irreducibility of $\bL$. Hence, $E \in \bLda(\bQ)$. 

    $\blacktriangleright$  Assume that for any question $\bQ = \langle \mbb{T}, \kappa, \D, x \rangle$, 
    $\bL(\bQ) \subseteq \bLda(\bQ)$. From Property~\ref{p-da-dw}, $\bL(\bQ) \subseteq \bLdw(\bQ)$. 
    From Theorem~\ref{th-cara}, $\bL$ satisfies Feasibility and Validity. 
    From Table~\ref{tab2}, $\bLda$ satisfies Irreducibility, hence so does $\bL$.
\end{proof}

\begin{theorem}
The explainer $\bLda$ satisfies Feasibility, Validity, Success and (Strong) Irreducibility. 
It violates the remaining axioms.
\end{theorem}

\begin{proof} Let us show the properties satisfied by ${\bLda}$.
	
	$\blacktriangleright$ \textit{Feasibility and Validity} of ${\bLda}$ follow from 
    the definition of the function. 

	$\blacktriangleright$ \textit{Success} of ${\bLda}$ follows from Property~\ref{p-da-dw} (every instance $x$ has a dwAXp, namely itself, and hence must have a subset-minimal dwAXp).

	$\blacktriangleright$ \textit{Irreducibility:} Irreducibility of $\bLda$ follows from 
    its minimality condition.

    $\blacktriangleright$ \textit{Coherence, Completeness, Monotonicity and CM} are violated due to their incompatibility with Irreducibility (see Theorem~\ref{th:incompatibility}).  
\end{proof}
\begin{theorem}
The function $\bLa$ satisfies all the axioms except Irreducibility and (Strong) Completeness. 
\end{theorem}

\begin{proof} Let us show the properties satisfied by ${\bLa}$.
	
	$\blacktriangleright$ \textit{Success, Strong Irreducibility, Coherence} of ${\bLa}$ have been shown in \cite{ecai23}.
	
	$\blacktriangleright$ \textit{Feasibility and Validity} of ${\bLa}$ follow from 
    the definition of the function.

    $\blacktriangleright$ \textit{Monotonicity and CM} follows straightforwardly from the fact 
    that for any question $\bQ = \langle\mbb{T}, \kappa, \D, x\rangle$, 
    $\bLa(\bQ) = \bLda(\bQ')$ where  $\bQ' = \langle\mbb{T}, \kappa, \mbb{F}(\mbb{T}), x\rangle$.

	$\blacktriangleright$ \textit{Irreducibility and Completeness:} To show their 
    violation by $\bLa$, consider the following theory, classifier, the dataset consisting of $x_1, x_2$ 
    and the question $\bQ = \langle\mbb{T},\kappa,\D, x_1\rangle$. 
    
    \begin{center}
		\begin{tabular}{c|cc|c}\hline
		           $\mbb{F}(\mbb{T})$   & $f_1$ & $f_2$ & $\kappa(x_i)$   \\\hline
	\rowcolor{maroon!10}		$x_1$  &   0   &  0    &  0 \\
	\rowcolor{maroon!10}		$x_2$  &   0   &  1    &  1 \\
                                $x_3$  &   1   &  0    &  1 \\
			                    $x_4$  &   1   &  1    &  0 \\\hline
		\end{tabular}
	\end{center}
    Note that $\bLa(\bQ) = \{\{(f_1,0),(f_2,0)\}\}$ and 
    $\nexists y \in \D$ such that $(f_2,0) \in y$ and $\kappa(y) \neq \kappa(x_1)$. 
    Hence, $\{(f_1,0),(f_2,0)\}$ can be reduced to $\{(f_2,0)\}$. The same example 
    shows that Completeness is violated (indeed, $\{(f_2,0)\} \notin \bLa(\bQ))$.

    $\blacktriangleright$ \textit{Strong Completeness:} Consider the following theory, 
    classifier, the dataset made of $x_1, x_2$ 
    and the question $\bQ = \langle\mbb{T},\kappa,\D, x_1\rangle$. 
    
    \begin{center}
		\begin{tabular}{c|cc|c}\hline
		           $\mbb{F}(\mbb{T})$   & $f_1$ & $f_2$ & $\kappa(x_i)$   \\\hline
	\rowcolor{maroon!10}		$x_1$  &   0   &  0    &  0 \\
	\rowcolor{maroon!10}		$x_2$  &   0   &  1    &  0 \\
                                $x_3$  &   1   &  0    &  1 \\
			                    $x_4$  &   1   &  1    &  1 \\\hline
		\end{tabular}
	\end{center}
    Note that $\bLa(\bQ) = \{\{(f_1,0)\}\}$. So, $x_1 \notin \bLa(\bQ)$ while 
    $\nexists y \in \mbb{F}(\mbb{T})$ such that $x_1 \subseteq y$ and 
    $\kappa(y) \neq \kappa(x_1)$.
\end{proof}



\begin{theorem}
If an explainer $\bL$ satisfies Feasibility, Success and Coherence, 
then for any question $\bQ$, $\bL(\bQ) \subseteq \bLdw(\bQ)$.
\end{theorem}

\begin{proof}
Let $\bL$ be an explainer that satisfies Feasibility, Success and Coherence. 
Let $\bQ = \langle\mbb{T},\kappa,\D, x\rangle$ be a question and $E \in \bL(\bQ)$. 
Feasibility leads to $E \subseteq x$. Let $Y = \{y\in \D \ \mid \  E \subseteq y\}$. 
Assume that $z \in Y$ and $\kappa(z) \neq \kappa(x)$. 
Let $\bQ' = \langle\mbb{T},\kappa,\D, z\rangle$ be 
a question. Success implies that $\bL(\bQ') \neq \emptyset$. So, $\exists E' \in \bL(\bQ')$. 
Feasibility ensures $E' \subseteq z$. Hence, $E \cup E' \subseteq z$, which means $E \cup E'$ is consistent. 
This contradicts Coherence which guarantees the inconsistency of $E \cup E'$.
\end{proof}

\begin{theorem}
An explainer $\bL$ satisfies Feasibility, Success and Coherence \textbf{iff} $\bL$ is a 
coherent explainer.
\end{theorem}

\begin{proof}
	Assume that an explainer $\bL$ satisfies Feasibility, Success and Coherence. 
	From Theorem~\ref{coh+success}, $\forall \bQ = \langle\mbb{T},\kappa,\D, x\rangle$, $\bL(\bQ) \subseteq \bLdw(\bQ)$.  
	From Coherence of $\bL$, the set $\bigcup\limits_{x \in \D} \bL(\langle\mbb{T},\kappa,\D, x\rangle)$ is coherent. 
	From Success, $\forall x \in \D$, $\bL(\langle\mbb{T},\kappa,\D, x\rangle) \neq \emptyset$. Hence, 
	$\bigcup\limits_{x \in \D} \bL(\langle\mbb{T},\kappa,\D, x\rangle) \in \mathtt{Coh}(\D,\kappa)$ and so $\bL$ is a 
	coherent explainer.  
	
	Let us now show that $\bLc$ satisfies the three axioms. Let $\bQ = \langle \mbb{T}, \kappa, \D, x \rangle$ be a question and $X = \bigcup\limits_{x \in \D}\bLc(\langle \mbb{T}, \kappa, \D, x \rangle)$.
	Since $X \in \mathtt{Coh}(\D,\kappa)$, then it is coherent, so $\bLc$ satisfies Coherence. 
	From Definition~7, $\exists E \in X$ such that $E \subseteq x$. So, $\bLc$ satisfies Success. 
	Feasibility is satisfied since by definition, every explanation is a subset of the instance.
\end{proof}

\begin{theorem}
Any coherent explainer $\bLc$ satisfies Feasibility, Validity, Success, Coherence. 
It violates Irreducibility and Completeness.
\end{theorem}
 
\begin{proof}
Theorem~\ref{cara:coh} shows that $\bLc$ satisfies Feasibility, Success, Coherence. 
From Proposition~\ref{links}, Validity follows from Feasibility, Success, Coherence. Theorem~\ref{th:incompatibility} shows that Coherence is incompatible 
with Irreducibility (set $\mathbf{I}_1$) and Completeness (set $\mathbf{I}_2$).
\end{proof}
\begin{theorem}
    The function $\bLt$ satisfies all the axioms except (Strong) Irreducibility and (Strong) Completeness. 
\end{theorem}

\begin{proof}
    $\ $ 
	
	\textit{Success, Feasibility, Coherence} follow from Theorem~\ref{cara:coh} 
     ($\bLt$ is a coherent explainer). Validity follows from Proposition~\ref{links}.	
	
	\textit{Monotonicity and CM} follow straightforwardly from the definition of $\bLt$ 
    which provides a single explanation (the instance itself) for every question.  

    \textit{Irreducibility and Completeness} are violated due to Theorem~\ref{th:incompatibility}. 

    \textit{Strong Irreducibility and Strong Completeness:} Consider the theory and classifier below. 
    \begin{center}
		\begin{tabular}{c|cc|c}\hline
			$\mbb{F}(\mbb{T})$ & $f_1$ & $f_2$ & $\kappa(x_i)$   \\\hline
\rowcolor{maroon!10}			$x_1$  &   0   &  0    &  0 \\
\rowcolor{maroon!10}			$x_2$  &   0   &  1    &  1 \\
            $x_3$  &   1   &  0    &  0 \\
			$x_4$  &   1   &  1    &  1 \\\hline
		\end{tabular}
	\end{center}
Consider the dataset $\D = \{x_1, x_2\}$. 
Note that $\bLt(\langle \mbb{T},\kappa,\D,x_1\rangle) = \{x_1\}$. 
So, $\{(f_2,0)\} \notin \bLt(\langle \mbb{T},\kappa,\D,x_1\rangle)$ while 
$\nexists y \in \mbb{F}(\mbb{T})$ such that $(f_2,0) \in y$ and 
$\kappa(x_1) \neq \kappa(y)$. Thus, Strong Completeness is violated. 
From the same example, note that for $l = (f_1,0)$, 
$x_1\setminus\{l\} = \{(f_2,0)\}$ while 
$\nexists y \in \mbb{F}(\mbb{T})$ such that $(f_2,0) \in y$ and 
$\kappa(x_1) \neq \kappa(y)$. Thus, Strong Irreducibility is violated. 
\end{proof}
\begin{theorem}
The function $\bLr$ satisfies Feasibility, Validity, Success and Coherence. 
It violates all the other axioms.
\end{theorem}

\begin{proof} 
    $\ $ 
	
	\textit{Success, Feasibility, Validity and Coherence} follow from Property~\ref{prop:lir} and Theorem~\ref{cara:coh}.
	
	

	\textit{Monotonicity:} Consider the sample $\D_1$ below. 
	\begin{center}
		\begin{tabular}{c|cc|c}\hline
			$\D_1$ & $f_1$ & $f_2$ & $\kappa(x_i)$   \\\hline
			$x_1$  &   0   &  0    &  0 \\
			$x_2$  &   1   &  0    &  1 \\\hline
		\end{tabular}
	\end{center}
	
	From $\D_1$, two dwAXp ($E_i$) are generated for $x_1$ and two others ($E'_j$) concern $x_2$. 
	\begin{itemize}
		\item [] $E_1 = \{(f_1,0)\}$    		\hfill  $E'_1 = \{(f_1,1)\}$
		\item [] $E_2 = \{(f_1,0), (f_2,0)\}$	\hfill  $E'_2 = \{(f_1,1), (f_2,0)\}$
	\end{itemize}
	Note that the four sets are not conflicting. Thus, $\mathtt{Irr}(\D_1,\kappa) = \{E_1, E_2, E'_1, E'_2\}$ and $\bLr(\langle\mbb{T},\kappa,\D_1,x_1\rangle) = \{E_1, E_2\}$. 
 
	\noindent Consider now the extended sample $\D_2 = \D_1 \cup \{x_3\}$ where $x_3$ is as shown below. 
	\begin{center}
		\begin{tabular}{c|cc|c}\hline
			$\ $    & $f_1$    & $f_2$     & $\kappa(x_3)$   \\\hline
			$x_3$     &   1      &  1        &  1              \\\hline
		\end{tabular}
	\end{center}
	Note that $\bLdw(\langle\mbb{T},\kappa,\D_2,x_1\rangle) = \bLdw(\langle\mbb{T},\kappa,\D_1,x_1\rangle)$,   
	$\bLdw(\langle\mbb{T},\kappa,\D_2,x_2\rangle) = \bLdw(\langle\mbb{T},\kappa,\D_1,x_2\rangle)$ and 
    $\bLdw(\langle\mbb{T},\kappa,\D_2,x_3\rangle) = \{E'_1, \{(f_2,1)\}, x_3\}$. 
    Note that $\{E_1, \{(f_2,1)\}\}$ is incoherent, then $E_1 \notin \mathtt{Irr}(\D_2,\kappa)$, 
    and so $E_1 \notin \bLr(\langle\mbb{T},\kappa,\D_2,x_1\rangle)$.
	
	\textit{Counter-Monotonicity:} To show that $\bLr$ violates Counter-Monotonicity, consider the sample $\D_1$ below. 
	\begin{center}
		\begin{tabular}{c|cc|c}\hline
			$\D_1$ & $f_1$ & $f_2$ & $\kappa(x_i)$   \\\hline
			$x_1$  &   1   &  0    &  0 \\
			$x_2$  &   0   &  1    &  1 \\\hline
		\end{tabular}
	\end{center}
	
	From $\D_1$, three dwAXp ($E_i$) are generated for $x_1$ and three others ($E'_j$) concern $x_2$. 
	\begin{itemize}
		\item [] $E_1 = \{(f_1,1)\}$    		\hfill  $E'_1 = \{(f_1,0)\}$
		\item [] $E_2 = \{(f_2,0)\}$			\hfill	$E'_2 = \{(f_2,1)\}$
		\item [] $E_3 = \{(f_1,1), (f_2,0)\}$	\hfill  $E'_3 = \{(f_1,0), (f_2,1)\}$
	\end{itemize}
	Note that $E_1$ and $E'_2$ are conflicting. The same holds for the pair $E'_1$ and $E_2$. 
	Hence, $\bLr(\langle\mbb{T},\kappa,\D_1,x_1\rangle) = \{E_3\}$. 
	Consider now the extended sample $\D_2 = \D_1 \cup \{x_3\}$ where $x_3$ is as shown below. 
	\begin{center}
		\begin{tabular}{c|cc|c}\hline
			$\ $    & $f_1$    & $f_2$     & $\kappa(x_3)$   \\\hline
			$x_3$     &   0      &  0        &  0              \\\hline
		\end{tabular}
	\end{center}
	Note that $\bLdw(\langle\mbb{T},\kappa,\D_2,x_1\rangle) = \bLdw(\langle\mbb{T},\kappa,\D_1,x_1\rangle)$ while  
	$\bLdw(\langle\mbb{T},\kappa,\D_2,x_2\rangle) = \{E'_2, E'_3\}$. 
	Then, $\bLr(\langle\mbb{T},\kappa,\D_2,x_1\rangle) = \{E_2, E_3\}$.

    \textit{Irreducibility and Completeness} are violated due to their incompatibility 
    with Coherence (sets $(\mathbf{I}_1)$ and $(\mathbf{I}_3)$ respectively in Theorem~\ref{th:incompatibility}). 

    \textit{Strong Irreducibility and Strong Completeness:} Consider the counter-example 
    given for Counter-Monotonicity, namely the dataset $\D_2$ and assume that $\kappa(x_4) = 0$. 
    Note that the set $\{\{(f_1,1)\}, \{(f_2,1)\}\}$ is incoherent under $(\D_2,\kappa)$. 
    Let $\bQ = \langle\mbb{T},\kappa,\D_2,x_1\rangle$. 
    So, $\bLr(\bQ) = \{E_2, E_3\}$. 
    Note that $E_1 \notin \bLr(\bQ)$ while $\nexists y \in \mbb{F}(\mbb{T})$ such that 
    $E_1 \subseteq y$ and $\kappa(y) \neq \kappa(x_3)$. So, Strong Completeness is 
    violated by $\bLr$. 

    Consider now the question $\bQ' = \langle\mbb{T},\kappa,\D_2,x_3\rangle$. 
    Note that $\bLr(\bQ') = \{x_3\}$. Let $l = (f_1,0)$. Note that 
    $\forall y \in \mbb{F}(\mbb{T})$ such that $x_3\setminus\{l\} \in y$, 
    $\kappa(y) = \kappa(x_3)$, so Strong Irreducibility is violated. 
\end{proof}


\begin{theorem} 
Let $\bQ = \langle\mbb{T},\kappa,\D,x\rangle$ be a question,
where $\kappa$ can be evaluated in polynomial time. 
Testing whether $E \in \bLr(\bQ)$ can be achieved in polynomial time.
\end{theorem}

\begin{proof}
Testing whether $E$ is a weak abductive explanation can be achieved in time $O(mn)$,
where $n$ is the number of features and $m$ the number of instances in the dataset $D$~\cite{ecai23}.
Testing whether a weak abductive explanation $E$ is irrefutable amounts to checking that no counter-example exists.
A counter-example is another instance $x' \in \D$ giving a distinct prediction
(i.e. $\kappa(x') \neq \kappa(x)$) together with a weak abductive explanation $E'$ of $x'$ such that
$E$ and $E'$ attack each other.
We claim that, for a given $x' \in \D$, we only need to test $E'=x'[F\setminus (ft(E) \cap \{i \mid x_i \neq x'_i\})]$
as a putative weak abductive explanation,
where $ft(E)$ is the set of features assigned a value in $E$ and $x_i$ is the
value of feature $i$ in vector $x$.

Clearly if this $E'$ (along with $x'$) provides a counter-example, then $E$ is not irrefutable.
For a given $x'$, suppose that $E'=x'[F\setminus (ft(E) \cap \{i \mid x_i \neq x'_i\})]$ 
is not a counter-example, but there is another $E''$ which is (for $x'$).
Since $E''$ is a counter-example (along with $x'$), $E''$ is a weak abductive explanation of $\kappa(x')=c'$. However,
by supposition,
$x'[F\setminus (ft(E) \cap \{i \mid x_i \neq x'_i\})]$ is not a weak abductive explanation of the decision $\kappa(x')=c'$, 
so there exists $u \in \D$ such that $\kappa(u)\neq \kappa(x')$ and $u[ft(E')]=x'[ft(E')]$. 
It follows that $ft(E'')$ cannot be a subset of $ft(E')$.
Let $j$ be a feature in $ft(E'')$ but not in $ft(E')$: hence $j$ must belong to  $ft(E) \cap \{i \mid x_i \neq x'_i\}$.
However, we know that $E$ and $E''$ attack each other, so they must agree on their intersection, which
is contradicted by the existence of $j$. Hence, we have a proof of the claim, by contradiction. 

The number of $x' \in \D$ to be tested is bounded by $m$, the size of the dataset $\D$.
For each such $x'$, the calculation of $E'=x'[F\setminus (ft(E) \cap \{i \mid x_i \neq x'_i\})$
is $O(n)$, where $n$ is the number of features, and testing whether this $E'$ is a weak abductive explanation
can be achieved in $O(mn)$ time. Hence, total complexity is $O(n^2m^2)$.
\end{proof}

\begin{theorem}
A surrogate explainer $\bLDT$ can be found in polynomial time.
\end{theorem}

\begin{proof}
This follows from the fact that standard algorithms, such as ID3, build
a decision tree in time which is polynomial in the size of the resulting tree
and the size $m$ of the dataset. We have seen in the proof of Proposition~\ref{p:existence} that the number of nodes in the decision tree is bounded by $m$.
\end{proof}
\begin{theorem}
The explainer $\bLsu$ satisfies Feasibility, Validity, Success and Coherence. 
It violates the remaining axioms.
\end{theorem}

\begin{proof}
Let $\bQ = \langle\mbb{T},\kappa,\D,x\rangle$, 
    $\bQ' = \langle\mbb{T},\sigma,\mathbb{F}(\mbb{T}),x\rangle$ and  
    $\sigma$ be a surrogate function on $(\mbb{T},\kappa, \D)$. 
    Recall that $x \in \D$.
    
From \cite{ecai23}, $\bLdw$ satisfies success when defined on the feature space. 
Hence, $\bLsu(\bQ) \neq \emptyset$ and so it satisfies Success. 

Let $E \in \bLsu(\bQ)$, then $E \in \bLdw(\bQ')$. By definition of 
a weak abductive explanation, $E \subseteq x$, so $\bLsu$ satisfies Feasibility. 
Furthermore, $\forall z \in \mathbb{F}(\mbb{T})$ such that $E \subseteq z$, 
$\sigma(z) = \sigma(x)$.  Since $\D \subseteq \mathbb{F}(\mbb{T})$, 
then $\forall z \in \D$, $\sigma(z) = \sigma(x)$. Since 
$\sigma$ is a surrogate function on $(\mbb{T},\kappa, \D)$, 
we have $\sigma(z) = \sigma(x) = \kappa(z) = \kappa(x)$.
So $\bLsu$ satisfies Validity. 

Coherence of $\bLsu$ follows from the Coherence of $\bLdw$ when defined on the feature space \cite{ecai23}.  

Violation of Irreducibility follows from violation of the property by $\bLdw$.

To show that $\bLsu$ violates Monotonicity and CM, consider the theory $\mbb{T}$ below, 
a classifier $\kappa$, its surrogate models $\sigma_1, \sigma_2$, and two datasets 
$\D_1 = \{x_1, x_2\}$ and 
$\D_2 = \{x_1, x_2, x_3\}$ on which the surrogate models have been trained respectively. 

\begin{center}
\begin{tabular}{c|cc|c|c|c}\hline
$\mbb{F}(\mbb{T})$ & $f_1$ & $f_2$ & $\kappa(x_i)$ & $\sigma_1(x_i)$ & $\sigma_2(x_i)$  \\\hline
			$x_1$  &   0   &  0    &  0            &   0             &  0  \\
			$x_2$  &   0   &  1    &  1            &   1             &  1  \\
            $x_3$  &   1   &  0    &  0            &   1             &  0  \\
			$x_4$  &   1   &  1    &  1            &   1             &  0  \\\hline
\end{tabular}
\end{center}

Let $\bQ_1 = \langle\mbb{T},\kappa,\D_1,x_1\rangle$, 
    $\bQ'_1 = \langle\mbb{T},\kappa,\D_2,x_1\rangle$, 
    $\bQ_2 = \langle\mbb{T},\kappa,\D_1,x_2\rangle$, 
    $\bQ'_2 = \langle\mbb{T},\kappa,\D_2,x_2\rangle$ be questions. 
Note that $\bLsu(\bQ_1) = \{x_1\}$ while $\bLsu(\bQ'_1) = \{x_1, \{(f_2,0)\}\}$, hence CM is violated. 
Note also that $\bLsu(\bQ_2) = \{x_2,  \{(f_2,1)\}\}$ while $\bLsu(\bQ'_2) = \{x_2\}$, hence Monotonicity is violated. 

Completeness and Strong Completeness are violated since 
$\{(f_2,0)\} \notin \bLsu(\bQ_1)$ while 
$\nexists y \in \D_1$ (resp. $\nexists y \in \mbb{F}(\mbb{T})$) such that 
$\{(f_2,0)\} \subseteq y$ and $\kappa(y) \neq \kappa(x_1)$.
\end{proof}

\end{appendix}

\bibliography{refsExplanations}
\end{document}